\documentclass{article}
\usepackage{ijcai25}

\usepackage{times}
\usepackage{soul}
\usepackage{url}
\usepackage[hidelinks]{hyperref}
\usepackage[utf8]{inputenc}
\usepackage[small]{caption}
\usepackage{graphicx}
\usepackage{amsmath}
\usepackage{amsthm}
\usepackage{booktabs}
\usepackage{algorithm}
\usepackage{algorithmic}
\usepackage[switch]{lineno}

\usepackage{wrapfig}
\usepackage{subfigure}
\usepackage{amsmath}
\usepackage{amssymb}
\hypersetup{
    colorlinks=True,
    linkcolor=blue,
    filecolor=blue,      
    urlcolor=black,
    citecolor=black,
}
\newtheorem{definition}{Definition}
\newtheorem{lemma}{Lemma}

\newtheorem{theorem}{Theorem}

\title{Reducing Action Space for Deep Reinforcement Learning via \\ Causal Effect Estimation}

\author{
Wenzhang Liu$^{1,2}$
\and
Lianjun Jin$^1$\and
Lu Ren$^{1}$\and
Chaoxu Mu$^{1}$\And
Changyin Sun$^{1,2}$\\
\affiliations
$^1$School of Artificial Intelligence, Anhui University\\
$^2$Engineering Research Center of Autonomous Unmanned System Technology, Ministry of Education\\
\emails
\{wzliu, penny\_lu\}@ahu.edu.cn,
ljjin@stu.ahu.edu.cn,
cxmu@tju.edu.cn,
cysun@seu.edu.cn
}
\begin{document}

\maketitle

\begin{abstract}
Intelligent decision-making within large and redundant action spaces remains challenging in deep reinforcement learning. Considering similar but ineffective actions at each step can lead to repetitive and unproductive trials. Existing methods attempt to improve agent exploration by reducing or penalizing redundant actions, yet they fail to provide quantitative and reliable evidence to determine redundancy. In this paper, we propose a method to improve exploration efficiency by estimating the causal effects of actions. Unlike prior methods, our approach offers quantitative results regarding the causality of actions for one-step transitions. We first pre-train an inverse dynamics model to serve as prior knowledge of the environment. Subsequently, we classify actions across the entire action space at each time step and estimate the causal effect of each action to suppress redundant actions during exploration. We provide a theoretical analysis to demonstrate the effectiveness of our method and present empirical results from simulations in environments with redundant actions to evaluate its performance. Our implementation is available at \url{https://github.com/agi-brain/cee.git}.
\end{abstract}

\section{Introduction}

In the past decade, deep reinforcement learning (DRL) has proven to be a powerful tool for tackling sequential decision-making tasks~\cite{mnih2015human,silver2017mastering,berner2019dota}. However, addressing the issue of action redundancy for DRL with large action spaces remains a challenge~\cite{baram2021action,Zhong_Yang_Zhao_2024}. This is because the generalization of neural networks to redundant actions can slow down the training process and lead to a waste of computational resources. As the action space grows, the presence of similar and redundant actions increases, reducing the probability of selecting optimal actions. Furthermore, these redundant actions often result in more repetitive and unproductive trials. Since many real-world decision-making problems involve large and redundant action spaces, overcoming the challenges posed by these redundant actions is crucial for effectively applying DRL in real scenarios~\cite{chandak2019learning,dhakate2022autonomous}.

Existing research suggests that masking out the redundant actions to reduce the action space can significantly improve the performance of DRL~\cite{vinyals2017starcraft,berner2019dota,ye2020mastering,huang2020closer}. Based on that, some works proposed to learn a minimal action space with sampled experiences. For example, Zahavy \textit{et al.}~\shortcite{zahavy2018learn} developed an action elimination network (AEN) that predicts invalid actions using signals provided by the environment as supervision, and masked out these actions during DRL training. Similarly, Chandak \textit{et al.}~\shortcite{chandak2019learning} introduced a structure for learning an action representation model that effectively reduces the original action space. Wu \textit{et al.}~\shortcite{wu2024efficient} transformed the learning of action masking for robotic palletization as a semantic segmentation task, and then embedded the learned action masking model into DRL training. Although these methods learn the reduced action space automatically, they rely on external knowledge of the environments, which limits their generalization.

In this paper, we explain how the one-step state transition can be modeled as a causal graph and identify potentially redundant actions based on their quantified causal effects to the next state distribution in the given context. Since the current state-action pair influences the next state distribution, it is feasible to evaluate the redundancy of each action using historical data or learned models. To achieve this, Baram \textit{et al.}~\shortcite{baram2021action} defined action redundancy using the Kullback-Leibler (KL) divergence between the next-state distributions conditioned on the state-action pair and the expected next-state distribution induced by the policy. The redundant actions could be avoided by minimizing the action redundancy. Zhong \textit{et al.}~\shortcite{Zhong_Yang_Zhao_2024} introduced a method that classified the actions across the whole action space via a similarity matrix. The actions with high similarity are deemed redundant and would be masked out during execution, thus allowing for a reduction of the action space. To measure the causal influence of actions, Seitzer \textit{et al.}~\shortcite{NEURIPS2021_c1722a79} calculated the conditional mutual information, which is related to the independence of the action and the next state given the current state. These works have provided valuable insights for our research. 

This paper introduces a novel method called causal effect estimation (CEE) to reduce the action space for DRL. We begin by defining the causal effect of each action, which is similar to the concept of action redundancy proposed by~\cite{baram2021action}. To simplify the calculation of the KL divergence for state distributions, we pre-train an inverse dynamics model to obtain a value network that provides equivalent measures of causal effects. This allows us to set an appropriate threshold for identifying redundant actions based on their causal effects. Moreover, it is found that combining CEE with action classification~\cite{Zhong_Yang_Zhao_2024} significantly improves performance. We support our method with theoretical analysis and extensive experimental results, demonstrating its effectiveness.

The contributions of our work are summarized as follows:

\begin{itemize}
    \item We introduce a causal graphical model to represent the next-state transitions in DRL and characterize potentially redundant actions based on their causal effects on the next-state distribution. This enables a reduction in the action space by systematically eliminating redundant actions, thereby improving exploration efficiency.
    \item We develop an efficient method to evaluate the causal effect of each action in the next-state transition. By pre-training a value network alongside an inverse dynamics model, we simplify the calculation of KL divergence, enabling direct and accurate estimation of each action's causal effect without requiring explicit environmental dynamics or external knowledge.
    \item We propose an action masking mechanism that integrates action classification results with causal effect estimation to filter out redundant actions during exploration. Our method demonstrates significant performance improvements across various tasks, as validated through extensive experimental evaluations.
\end{itemize}

\section{Related Work}

\subsubsection{Eliminating Redundant Actions in DRL}

Existing research has shown that reducing redundant actions can accelerate the learning process. For instance, Dulac-Arnold \textit{et al.}~\shortcite{dulac2015deep} embedded prior action information into a continuous space, achieving sub-linear complexity relative to the number of actions. Zahavy \textit{et al.}~\shortcite{zahavy2018learn} combined the DRL algorithm with an AEN module to predict and eliminate invalid actions. More directly, using action masks provided by the environment is a commonly used trick to remove redundant actions, such as StarCraft~II~\cite{vinyals2017starcraft}, Dota2~\cite{berner2019dota}, and Honor of Kings~\cite{ye2020mastering}. Huang \textit{et al.}~\shortcite{huang2021gym}~proposed the Gym-µRTS library, demonstrating that action masks enhance the performance of real-time strategy (RTS) games. Additionally, Kanervisto \textit{et al.}~\shortcite{kanervisto2020action} and Ye \textit{et al.}~\shortcite{ye2020mastering} conducted ablation studies showing that masking invalid actions is crucial for improving performance. However, these methods become increasingly limited as the invalid action space expands. Huang and Onta{\~{n}}{\'{o}}n~\shortcite{huang2020closer} provided a detailed examination of invalid action masking in policy-based DRL algorithms, demonstrating that this technique empirically scales well as the number of invalid actions increases. Zhong \textit{et al.}~\shortcite{Zhong_Yang_Zhao_2024} employed a redundant action classification mechanism to construct a minimal action space to mask redundant actions and improve exploration efficiency. In this paper, we extend the above ideas and introduce a quantitative measurement of actions' redundancy from a causal perspective.

\subsubsection{Estimating the Causal Effects of Actions}

Recent research has shown that sequential decision-making tasks can be represented using causal graphical models, which makes it increasingly common to estimate how actions causally influence state transition. Corcoll and Vicente~\shortcite{DBLP:journals/corr/abs-2010-01351} proposed an approach that explores random effects and uses counterfactual reasoning to identify the causal impact of actions. To improve the sample efficiency of DRL, Pitis \textit{et al.}~\shortcite{pitis2020counterfactual} used an attention-based method to discover local causal structures in the decoupled state space, thereby enhancing the predictive performance of the model. Madumal \textit{et al.}~\shortcite{Madumal_Miller_Sonenberg_Vetere_2020} introduced an action influence model to provide causal explanations for the behavior of model-free DRL agents. Seitzer \textit{et al.}~\shortcite{NEURIPS2021_c1722a79} introduced the influence situation-dependent causal influence metric based on conditional mutual information to measure the causal impact of actions on the environment. 
Besides, Zhu \textit{et al.}~\shortcite{zhu2022invariant} induced the causation between actions and states by taking the invariance of action effects among the states. Herlau and Larsen~\shortcite{herlau2022reinforcement} identified the causal variables by maximizing the natural indirect effects and developed effective causal representations related to DRL agents. In scenarios with sparse rewards and high action noise, Zhang \textit{et al.}~\shortcite{NEURIPS2023_402e1210} measured the causal relationship between state and actions, and developed an interpretable reward redistribution methods. Nguyen \textit{et al.}~\shortcite{nguyen2024variableagnosticcausalexplorationreinforcement} proposed a variable-agnostic causal exploration method to discover the causal relationships among states and actions. Unlike these methods, our approach (i) evaluates the causal effects of actions using a value network that is pre-trained alongside an inverse dynamics model, which replaces the computation of KL divergence, (ii) masks out the actions with low causal effects after classification, leading to significant performance improvements.
 
\section{Background}

We focus on sequential decision-making tasks that can be modeled as a Markov decision process (MDP). An MDP is defined as a tuple $(\mathcal{S}, \mathcal{A}, P, R, \gamma)$, where $\mathcal{S}$ is the state space, $\mathcal{A}$ is the action space, $P: \mathcal{S} \times \mathcal{A} \times \mathcal{S} \to [0, 1]$ denotes the transition probability function, $R: \mathcal{S} \times \mathcal{A} \times \mathcal{S} \to \mathbb{R}$ is the reward function, and $\gamma \in (0, 1)$ is the discount factor. A DRL agent interacts with the environment using a policy $\pi$ that generates an action $a_t \in \mathcal{A}$ given $s_t \in \mathcal{S}$. The environment then transitions to the next state $s_{t+1} \in \mathcal{S}$ and provides a reward $r_{t+1}$. The goal of a DRL algorithm is to find an optimal policy $\pi^*$ that maximizes the cumulative discounted rewards over time: $J(\pi) = \Sigma^{\infty}_{t=0}{\mathbb{E}_{(s_t, a_t) \sim \rho_{\pi}}[\gamma^t r_{t+1}]}$, where $\rho_{\pi}$ denotes the trajectory distribution induced by $\pi(a_t|s_t)$.

\subsection{Policy with Action Mask}

This paper focuses on the discrete action spaces, where actions can be enumerated individually: $\mathcal{A}=\{A_1, \dots, A_N\}$, $N=|\mathcal{A}|$ is the size of action space. In a standard policy gradient-based algorithm, given state $S \in \mathcal{S}$, the policy can be parameterized by $\theta = [l_1, l_2, \dots, l_N]$:

\begin{equation}
    \pi_{\theta}(\cdot | S) = \text{softmax}([l_1, l_2, \dots, l_N]),
\end{equation}
where
\begin{equation}
    \pi_{\theta}(A_i | S) = \frac{\exp{(l_i)}}{\Sigma^{N}_{j=1}{\exp{(l_j)}}}.
\end{equation}
Let $\boldsymbol{m}=[m_1, m_2, \dots, m_N]^{\top}$ be a masking vector, in which
\begin{equation}
    m_i = \left\{
    \begin{aligned}
        & 1, & & \text{if } A_i \text{ is available,} \\
        & 0^{+}, & &\text{if } A_i \text{ is eliminated.}
    \end{aligned}
    \right.
\label{eq:masking_vector_definition}
\end{equation}
Here, $m_i = 0^{+}$ means $m = \lim_{x \to 0^{+}}{x}$. Thus, we can get a masked policy of $\pi_{\theta}$ induced by $\boldsymbol{m}$:
\begin{equation}
    \pi^m_{\theta}(\cdot | S) = \text{softmax}([l_1 + \log{(m_1), \dots, l_N + \log{(m_N)}]}),
\label{equation:masked_policy}
\end{equation}
and specifically, $\forall A_i \in \mathcal{A}$, we have
\begin{equation}
    \pi^m_{\theta}(A_i | S) = \left\{
    \begin{aligned}
        & \pi_{\theta}(A_i | S), & & m_i = 1, \\
        & 0, & & m_i = 0^{+}.
    \end{aligned}
    \right.
\end{equation}

As analyzed in~\cite{huang2020closer}, the policy gradient produced by the masking vector $\boldsymbol{m}$ remains valid. Therefore, the masked policy in Eq.(\ref{equation:masked_policy}) will not disrupt the training process.

\subsection{Causal Graphical Model for State Transition}

\begin{wrapfigure}{r}{0cm}
\centering
\includegraphics[width=0.2\textwidth]{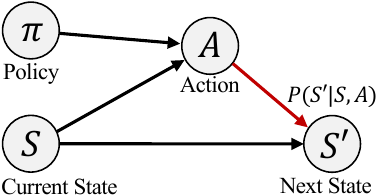}
\caption{Causal graphical model for state transition.}
\label{fig:causal_graph}
\end{wrapfigure}

The one-step transition from state $S \in \mathcal{S}$ to $S' \in \mathcal{S}$ with policy $\pi$ can be described by a causal graphical model (CGM)~\cite{scholkopf2022causality} over the set of nodes $\mathcal{V}=\{S, \pi, A, S'\}$ (as illustrated in Figure~\ref{fig:causal_graph}). Denote $\text{Pa}(V_i)$ as the set of all parent nodes of $V_i \in \mathcal{V}$. The stochastic intervention on node $V_i$ can be built by replacing its distribution $p(V_i|\text{Pa}(V_i))$ as $q(V_i|\text{Pa}(V_i))$, denoted as $\text{do}(V_i:=q(V_i|\text{Pa}(V_i)))$~\cite{NEURIPS2021_c1722a79}. Based on that, we introduce the definition of causal actions.

\begin{definition}[Causal Actions]
    The action $A \in \mathcal{A}$ is a cause of state $S'$ given state $S$, if and only if $\exists \pi_1, \pi_2$, $\pi_1 \neq \pi_2$, $P(S'|S, \text{do}(A:=\pi_1(A|S))) \neq P(S'|S, \text{do}(A:=\pi_2(A|S)))$. $A$ is called a causal action from state $S$ to $S'$. 
\label{definition:causal-action}
\end{definition}

A subset of action space that contains all causal actions from $S$ to $S'$ is called a local causal action space, denoted as $\mathcal{A}_{S \to S'} \subseteq \mathcal{A}$. If action $A$ is a cause of state $S'$, there is an edge from $A$ to $S'$ in the CGM: $A \longrightarrow S'$. We denote the CGM by $\mathcal{G}(\mathcal{V}, \mathcal{E})$, which consists of a set of nodes $\mathcal{V}$ and a set of all edges $\mathcal{E}$. It should be noticed that $\mathcal{G}(\mathcal{V}, \mathcal{E})$ must be a directed acyclic graph. In the following sections, we will show how to quantitatively evaluate the causality of actions $A_i \in \mathcal{A}$ (for $i=1,\dots, N$), and how to leverage this evaluation to reduce the action space for the DRL agent. 

\section{Method}

The framework of our proposed method is illustrated in Figure~\ref{fig:algorithm}. We start by introducing a measurement to estimate the causal effects of actions. Next, we describe how the action space can be efficiently reduced by leveraging the causal effects across the entire action set. Furthermore, we introduce a grouping strategy that categorizes actions into distinct groups and selects those with causal effects exceeding a predefined threshold within each group. Additional implementation details are provided to enhance clarity.

\begin{figure*}[t]
\centering
\includegraphics[width=0.8\textwidth]{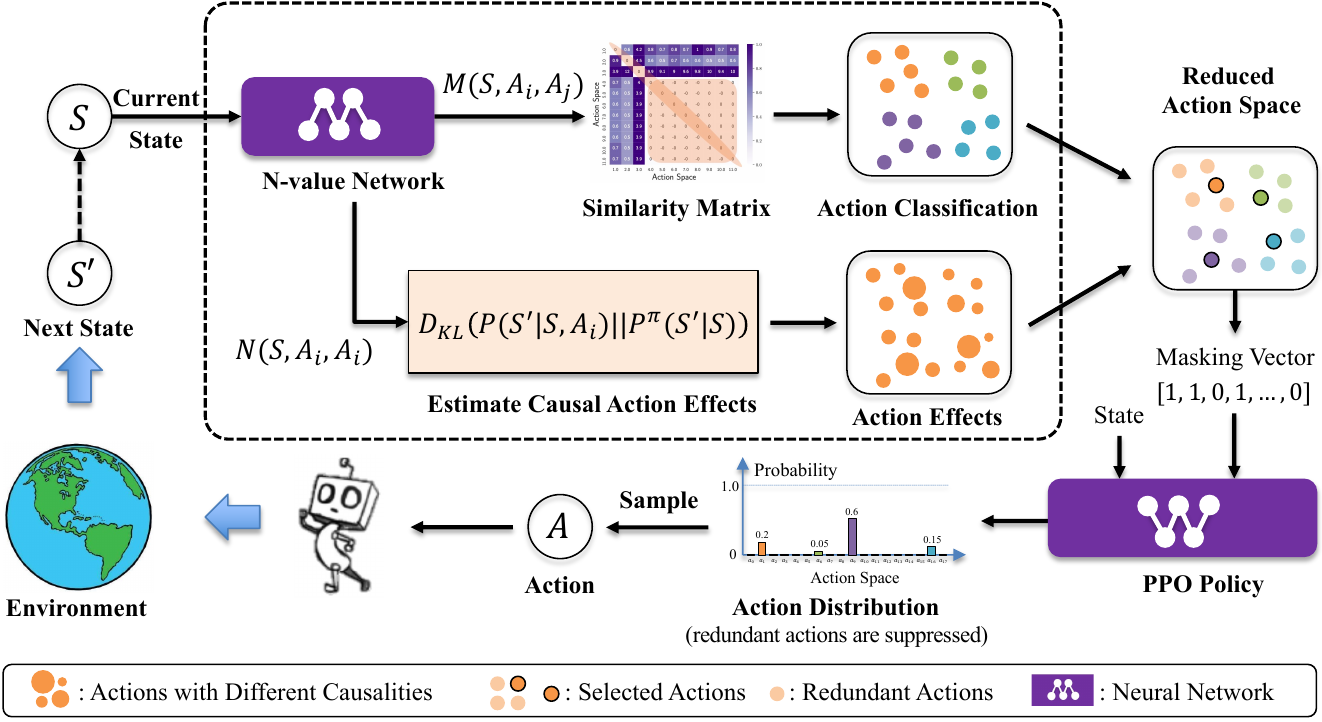}
\caption{The framework of our method.}
\label{fig:algorithm}
\end{figure*}

\subsection{Causal Effects Estimation for Actions}

According to the definition of causal actions, we can determine the existence of causality between action $A \in \mathcal{A}$ and next state $S' \in \mathcal{S}$ by intervening on the distribution of $A$ and observing whether the probability distribution of the $S'$ changes. We describe the relationship between them based on conditional probability independence. If the intervention of action $A$ in $S$ never changes the probability distribution of state $S'$ (i.e., it is not a causal action), then the action $A$ and $S'$ are considered conditionally independent: $S' \!\perp\!\!\!\perp A | S$. In other words, there is an edge from $A$ to $S'$ given $S$ for a causal graph $\mathcal{G}$ if and only if $S' \not\!\perp\!\!\!\perp A | S$ \cite{NEURIPS2021_c1722a79}. This dependence can be measured by the conditional mutual information. Thus, we introduce the definition of measurement for causal effects of actions:

\begin{definition}[Causal Effects of Actions]
    If an action $A \in \mathcal{A}$ is a cause of $S' \in \mathcal{S}$ given $S \in \mathcal{S}$ and policy $\pi$, we define the causal effects of $A$ by the KL divergence:
    \begin{equation}
        C^{\pi}(A | S \rightarrow S') = D_{KL}\left(P(S'|S, A) \| P^{\pi}(S'|S)\right).
    \label{eq:definition_of_causality}
    \end{equation}
\label{definition:causal_effects_of_actions}
\end{definition}

Form the non-negativity of KL divergence, we can infer that $C^{\pi}(A | S \rightarrow S') >= 0$. Next, we will discuss how to use the measurement of causal effects in Definition~\ref{eq:definition_of_causality} to determine whether an action is causal. Before that, we fist present a theorem (proof in Appendix):

\begin{theorem}
    Consider a causal graphical model $\mathcal{G}$ induced by a set of nodes $\mathcal{V}=\{S, \pi, A, S'\}$ and the set of edges $\mathcal{E}$. $A \in \mathcal{V}$ is a cause of $S'$ given $S$ (i.e. there is an edge from $A$ to $S'$), if and only if $C^{\pi}(A | S \rightarrow S') > 0$, otherwise, $C^{\pi}(A | S \rightarrow S') = 0$. 
\label{theorem:causal_effects}
\end{theorem}

Theorem~\ref{theorem:causal_effects} provides a strict condition for determining whether an action is causal. In practice, a threshold $\tau > 0$ can be used to select the causal actions, forming a reduced action space known as the approximate causal action space:

\begin{definition}[Approximate Causal Action Space]
    Given state pairs $S, S' \in \mathcal{S}$ and a discrete action space $\mathcal{A} = \{A_1, \dots A_N\}$, the approximate causal action space $\mathcal{A}_{S \to S', \tau}$ is defined as
    \begin{align}
        & \mathcal{A}_{S \to S', \tau} \subseteq \mathcal{A}, \label{definition:approximate-calsal-effects-eq1}\\
        & \text{s.t.  } \forall A_i \in \mathcal{A}_{S \to S', \tau}, C^{\pi}(A_i | S \rightarrow S') > \tau.\label{definition:approximate-calsal-effects-eq2}
    \end{align}
\label{definition:approximate-calsal-effects}
\end{definition}

Then, we can set the masking vector $\boldsymbol{m} = [m_1, \dots, m_N]^{\top}$ similar to Eq.(\ref{eq:masking_vector_definition}) as

\begin{equation}
    m_i = \left\{
    \begin{aligned}
        & 1, & & \text{if } A_i \in \mathcal{A}_{S \to S', \tau}, \\
        & 0^{+}, & &\text{if } A_i \notin \mathcal{A}_{S \to S', \tau}.
    \end{aligned}
    \right.
\label{eq:masking-approximal-causal-action-space}
\end{equation}
Next, a masked policy $\pi^m_{\theta}$ can be induced by this $\boldsymbol{m}$. $\pi^m_{\theta}$ will be used as the exploration policy, enabling the agent to interact with the environment while avoiding ineffective actions based on the measurement of the causal effects.

\subsection{Estimating Causal Effects of Actions Using Modified Inverse Dynamic Model}

Calculating the KL divergence for state distributions becomes challenging when the state space is large or continuous. To address this complexity in evaluating $C^{\pi}(A | S \rightarrow S')$, we propose an equivalent approach to efficiently estimate the causal effect of actions, thereby reducing computational overhead.

\begin{lemma}
    Define the N-value for $A_i, A_j \in \mathcal{A}$ as
    \begin{equation}
        N(S, A_i, A_j) = \mathbb{E}_{S' \sim P(\cdot|S, A_i)}{\left[ \log{\frac{P^{\pi}(A_j|S, S')}{\pi(A_j|S)}} \right]},
    \label{definition:N-value}
    \end{equation}
    then, the calculation in Eq.(\ref{eq:definition_of_causality}) can be expressed as
    \begin{equation}
        C^{\pi}(A | S \rightarrow S') = N(S, A, A).
        \label{eq:causal-effects-estimation-via-N-values}
    \end{equation}
\label{lemma:causal_action_effect_estimation}
\end{lemma}
The $P^{\pi}(A_j|S, S')$ is the inverse dynamic model that can significantly reduce the calculation compared with the state transition model $P(S'|S, A_j)$. Please refer to Appendix for the proof of Lemma~\ref{lemma:causal_action_effect_estimation}. 

In this paper, we follow the method of no prior mask (NPM)~\cite{Zhong_Yang_Zhao_2024} to pre-train an inverse dynamics model and an N-value network for further use. Once the N-value network is trained, we can evaluate the causal effects more directly during exploration. 

\subsection{Actions Classification}

Although eliminating actions with no causal effect can enhance performance, the number of such actions is relatively small, limiting the overall reduction in the action space. To further optimize the action space, we propose a method that groups actions into clusters and selects those with causal effects exceeding a given threshold from each cluster, thereby constructing a smaller and more focused causal action space. 

To cluster the actions into groups, we define a similarity factor $M(S, A_i, A_j)$ for action $A_i$ and $A_j$ in state $S$, expressed as:
\begin{equation}
    M(S, A_i, A_j) = D_{KL}(P(S'|S, A_i) \| P(S'|S, A_j)).
\end{equation}
As is introduced in~\cite{Zhong_Yang_Zhao_2024}, this similarity factor can be further calculated using
\begin{equation}
    M(S, A_i, A_j) = N(S, A_i, A_i) - N(S, A_i, A_j).
\end{equation}
We use the state-dependent action cluster $\mathcal{A}^{k}_{S, \epsilon}$ to represent the sets of different actions groups:
\begin{align}
    & \mathcal{A}^{k}_{S, \epsilon} \subseteq \mathcal{A}, \text{ and } \cup_{k}{\mathcal{A}^{k}_{S, \epsilon}} = \mathcal{A},  
    \label{eq:action-cluster}\\
    & \text{s.t.  } \forall A_i, A_j \in \mathcal{A}^{k}_{S, \epsilon}, M(S, A_i, A_j) < \epsilon.
    \label{eq:action-cluster-constraints}
\end{align}
By combing the approximate causal action space in Definition~\ref{definition:approximate-calsal-effects} and the state-dependent action cluster in Eq.(\ref{eq:action-cluster})-(\ref{eq:action-cluster-constraints}), we can create the final reduced action space for exploration. 

\subsection{Minimal Causal Action Space}

Although the actions were grouped based on similarity factors from the N-value network, the causal effects within the same group vary, making it difficult to establish a uniform threshold $\tau$ for action selection. Therefore, we gather the causal effects $C^{\pi}(A | S \rightarrow S')$ of actions within group $k$ and normalize them using the softmax operator with temperature $T$, which we refer to as the relative causal effects of actions:

\begin{definition}[Relative Causal Effects of Actions]
    Given a group of actions $\mathcal{A}^{k}_{S, \epsilon}$ in (\ref{eq:action-cluster}). The causal effect of action $A^k \in \mathcal{A}^{k}_{S, \epsilon}$ is measured as $C^{\pi}(A^k|S \rightarrow S')$. Then its relative causal effects in group $k$ can be defined as
    \begin{equation}
        C^{\pi}_{R}(A^k | S \rightarrow S') = \frac{C^{\pi}(A^k | S \rightarrow S') / T}{\Sigma_{A^k_i \in \mathcal{A}^{k}_{S, \epsilon}}{C^{\pi}(A^k_i | S \rightarrow S') / T}}.
    \label{eq:definition_of_relative_causality}
    \end{equation}
\end{definition}
By calculating the relative causal effects of actions in each group, we construct a minimal causal action space by selecting the actions with $C^{\pi}_{R} > \tau$, where $\tau \in (0, 1)$. This minimal causal action space can be described as
\begin{align}
    & \mathcal{A}^{\text{min}}_{S \to S', \tau} \subseteq \mathcal{A}, \label{definition:minimal-calsal-effects-eq1}\\
    & \text{s.t.  } \forall A^k \in \mathcal{A}^{\text{min}}_{S \to S', \tau}, C^{\pi}_{R}(A^k | S \rightarrow S') > \tau.
    \label{definition:minimal-calsal-effects-eq2}
\end{align}
Similar to Eq.(\ref{eq:masking-approximal-causal-action-space}), a masking vector $\boldsymbol{m}$ can be induced by the minimal causal action space. After that, the agent can explore in the environment with the masked policy $\pi^{m}_{\theta}$.

Our method is composed of two phases. In the first phase, we pre-train the inverse dynamics models and the N-value network. The second phase focuses on estimating the causal effects of actions to reduce the action space, as illustrated in Figure~\ref{fig:algorithm}. The detailed procedure of our method is provided in Algorithm~\ref{algorithm:cee}.

\begin{figure*}[t]
\centering
\subfigure[]{
\includegraphics[height=0.28\columnwidth]{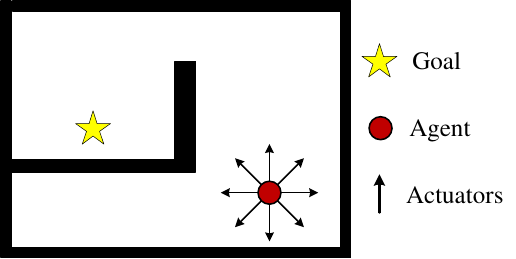}
\label{fig:Maze}}
\hspace{0.05\linewidth}
%\hfill
\subfigure[]{
\includegraphics[height=0.30\columnwidth]{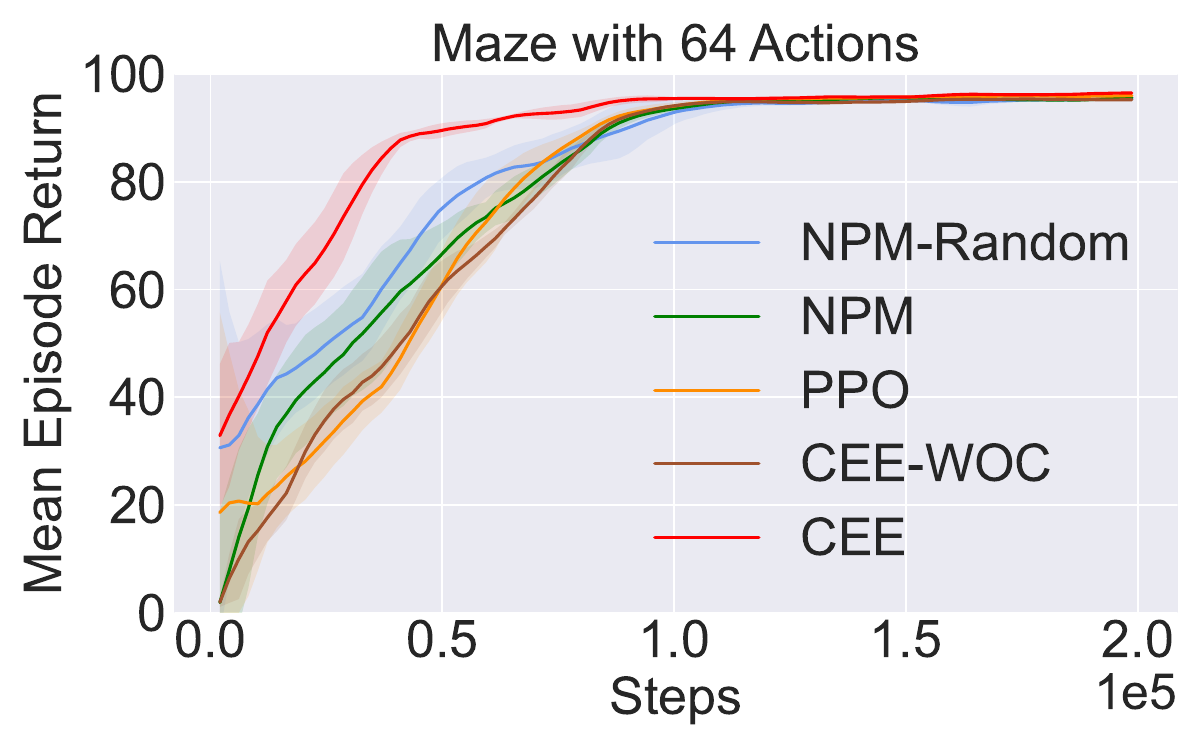}
\label{fig:results:maze_64_actions}}
\hspace{0.05\linewidth}
%\hfill
\subfigure[]{
\includegraphics[height=0.30\columnwidth]{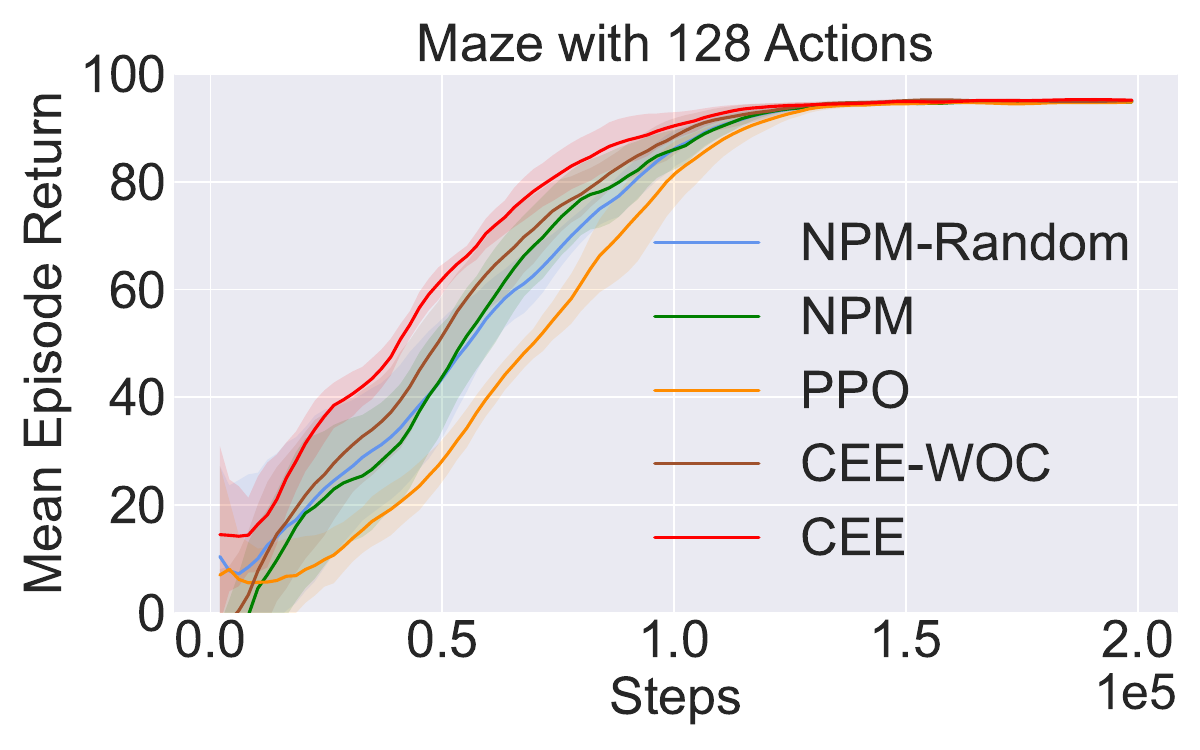}
\label{fig:results:maze_128_actions}}
\caption{Illustration of the Maze environment and the results: (a) Every actuator has two states: on and off. The agent's movement is determined by the vector sum of all the actuators that are on. The agent's task is to reach the goal; (b) Results on Maze environment with 64 actions; (c) Results on Maze environment with 128 actions.}
\label{fig:results}
\end{figure*}

\begin{figure*}[t]
\centering
\subfigure[]{
\includegraphics[height=0.30\columnwidth]{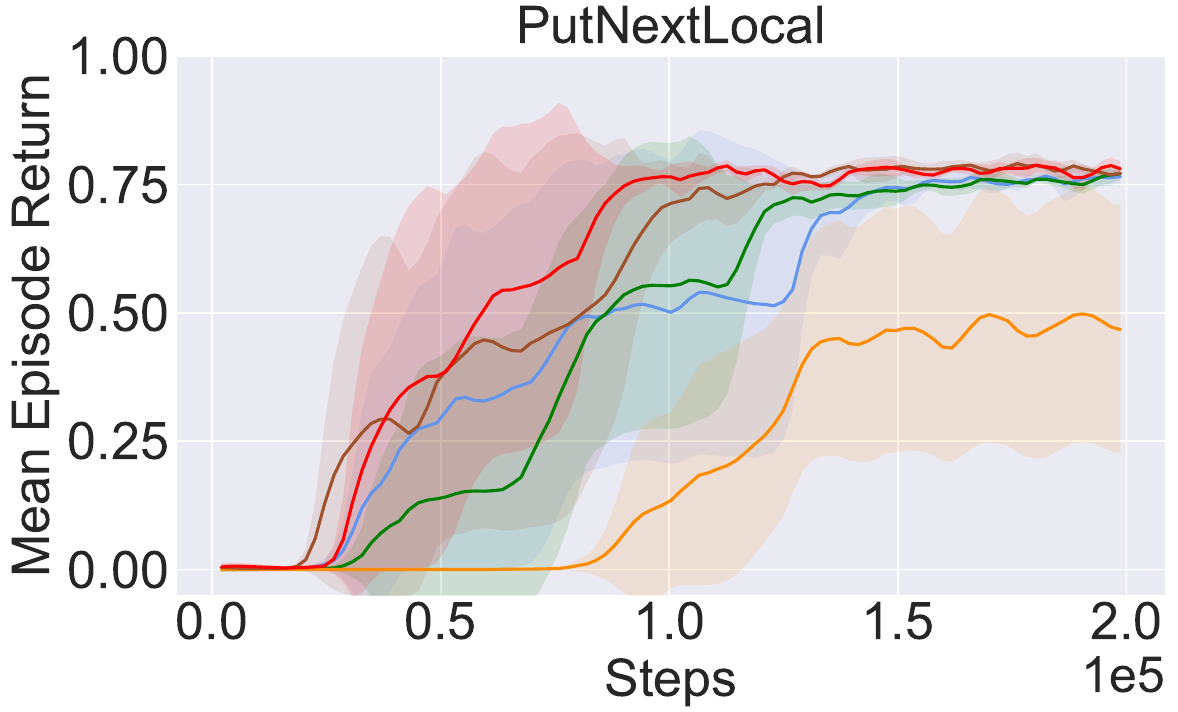}
\label{fig:results:minigrid_1}}
\hfill
\subfigure[]{
\includegraphics[height=0.30\columnwidth]{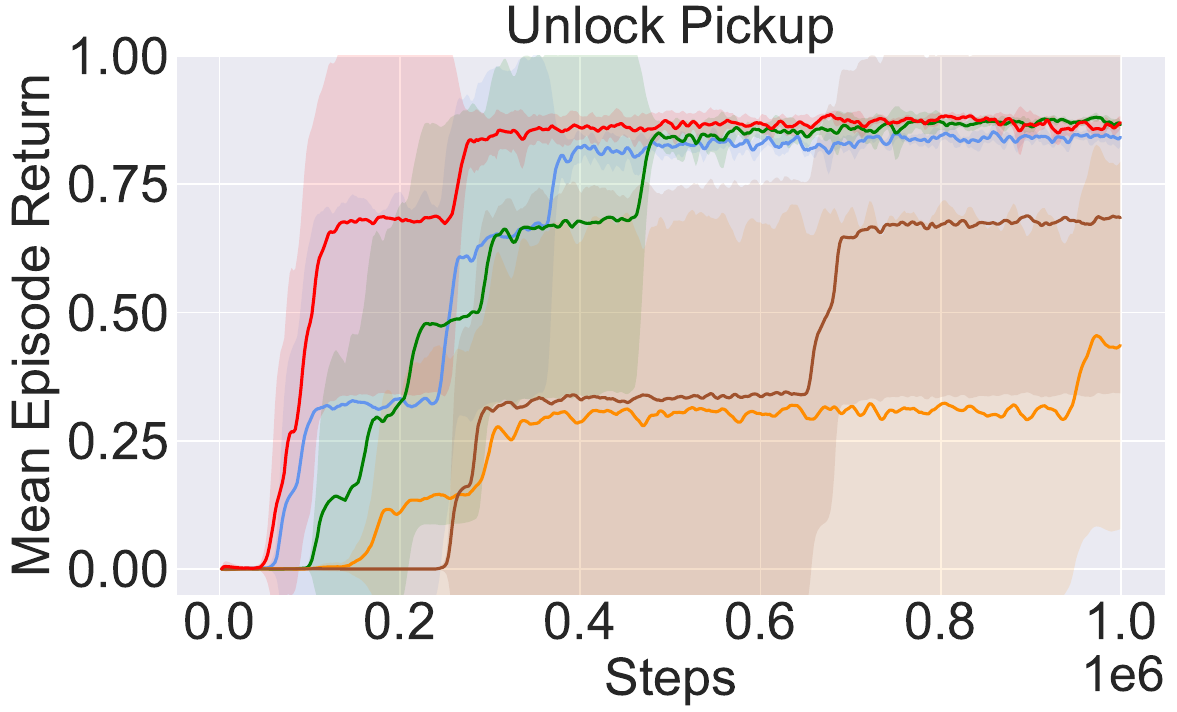}
\label{fig:results:minigrid_2}}
\hfill
\subfigure[]{
\includegraphics[height=0.30\columnwidth]{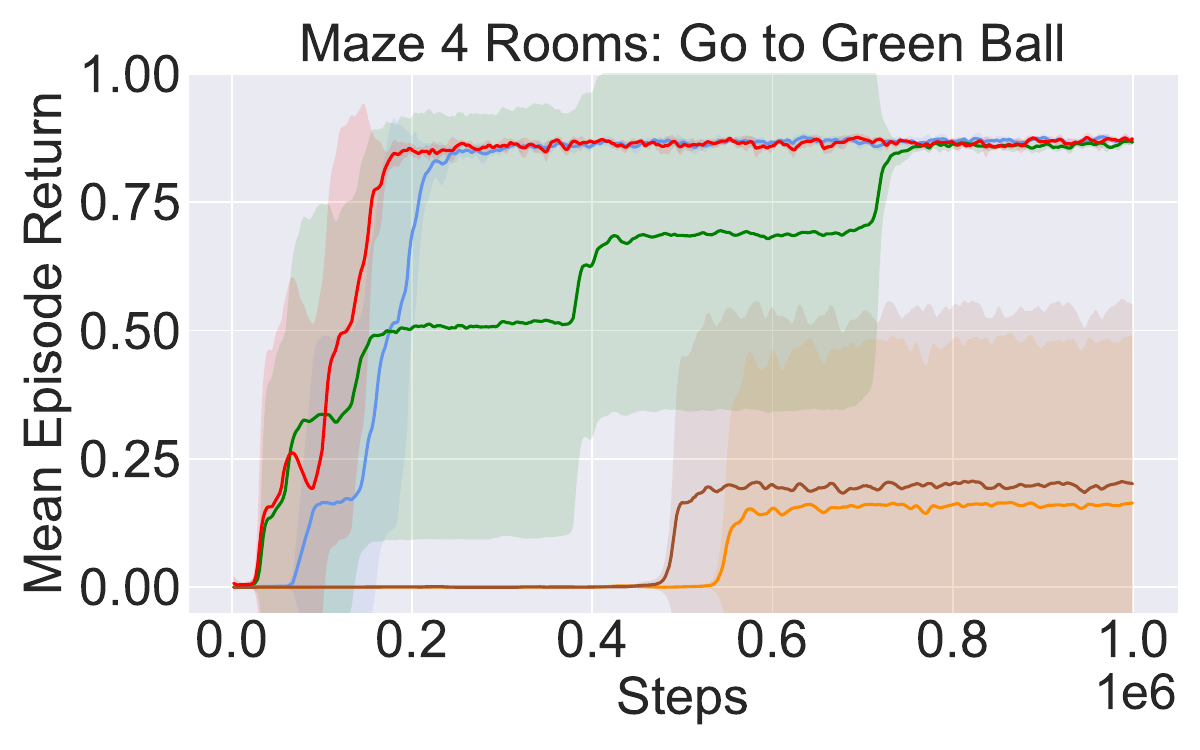}
\label{fig:results:minigrid_3}}
\hfill
\subfigure[]{
\includegraphics[height=0.30\columnwidth]{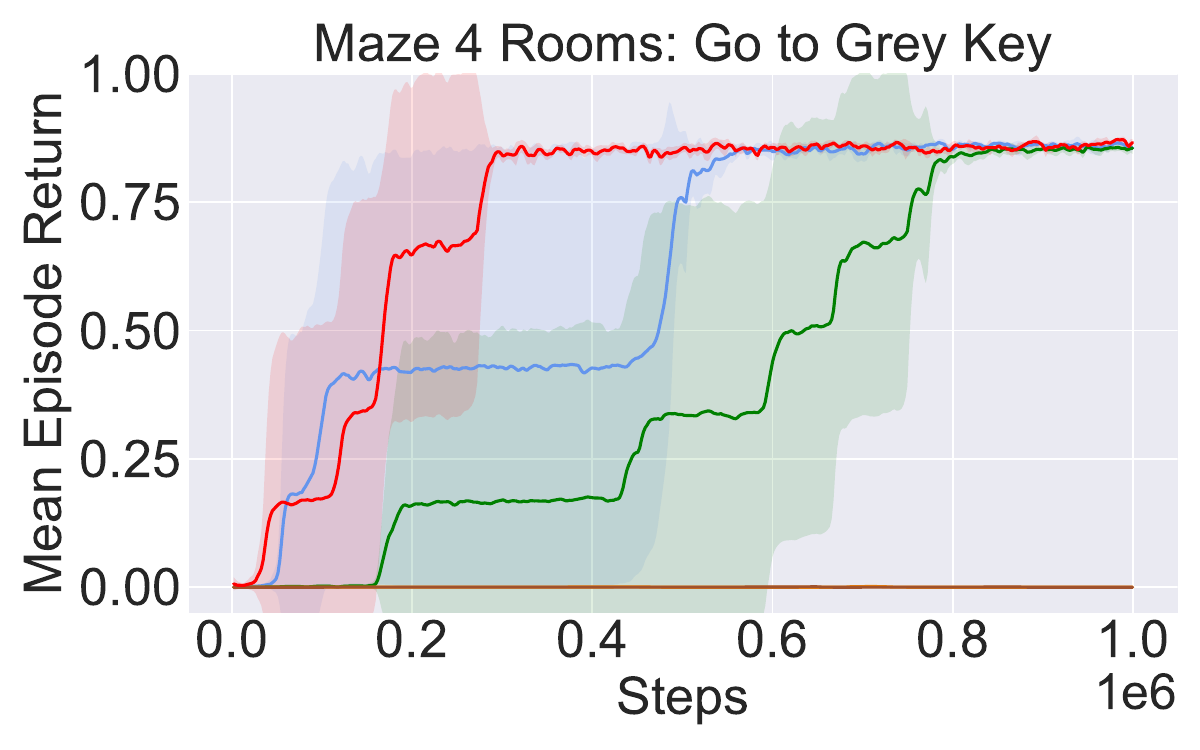}
\label{fig:results:minigrid_4}}
\hfill
\subfigure[]{
\includegraphics[height=0.30\columnwidth]{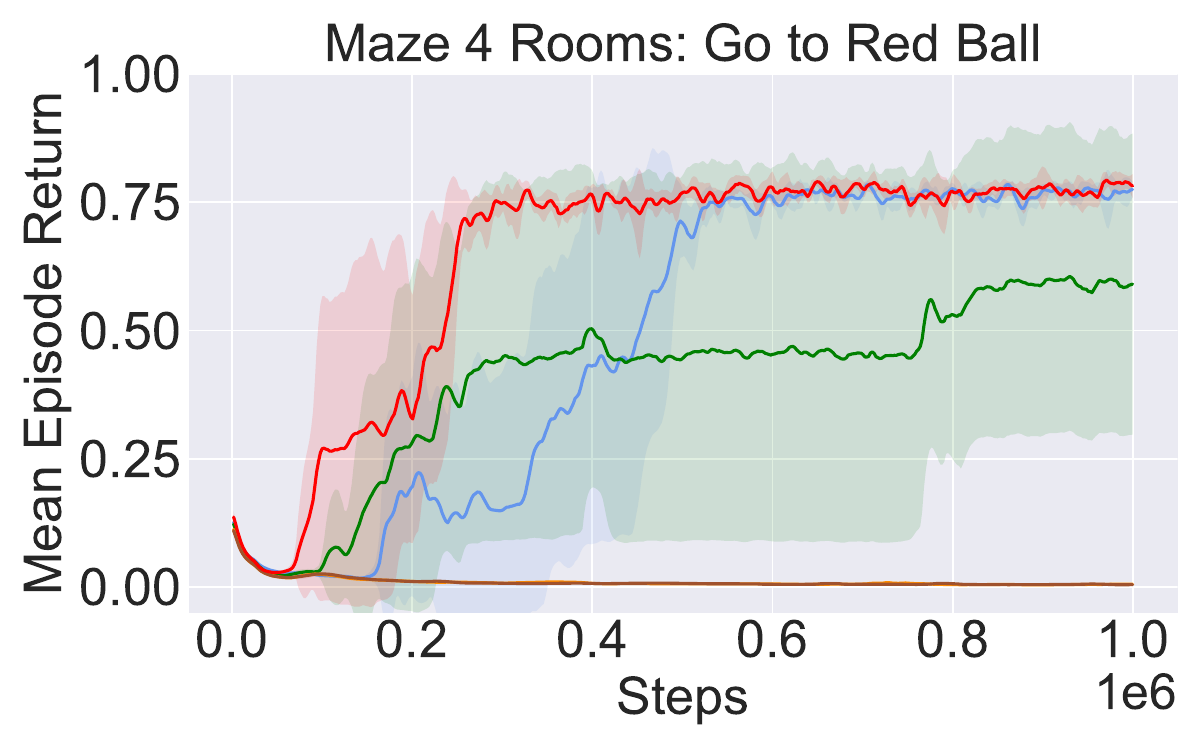}
\label{fig:results:minigrid_5}}
\hfill
\subfigure[]{
\includegraphics[height=0.30\columnwidth]{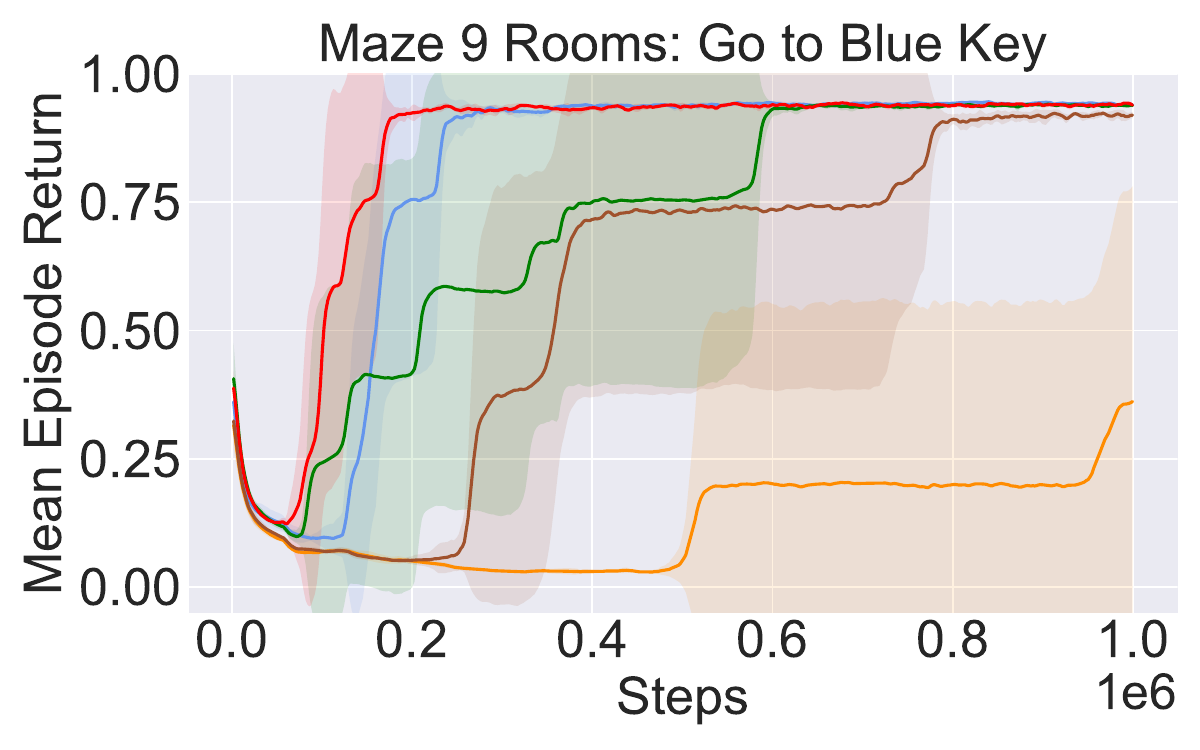}
\label{fig:results:minigrid_6}}
\hfill
\subfigure[]{
\includegraphics[height=0.30\columnwidth]{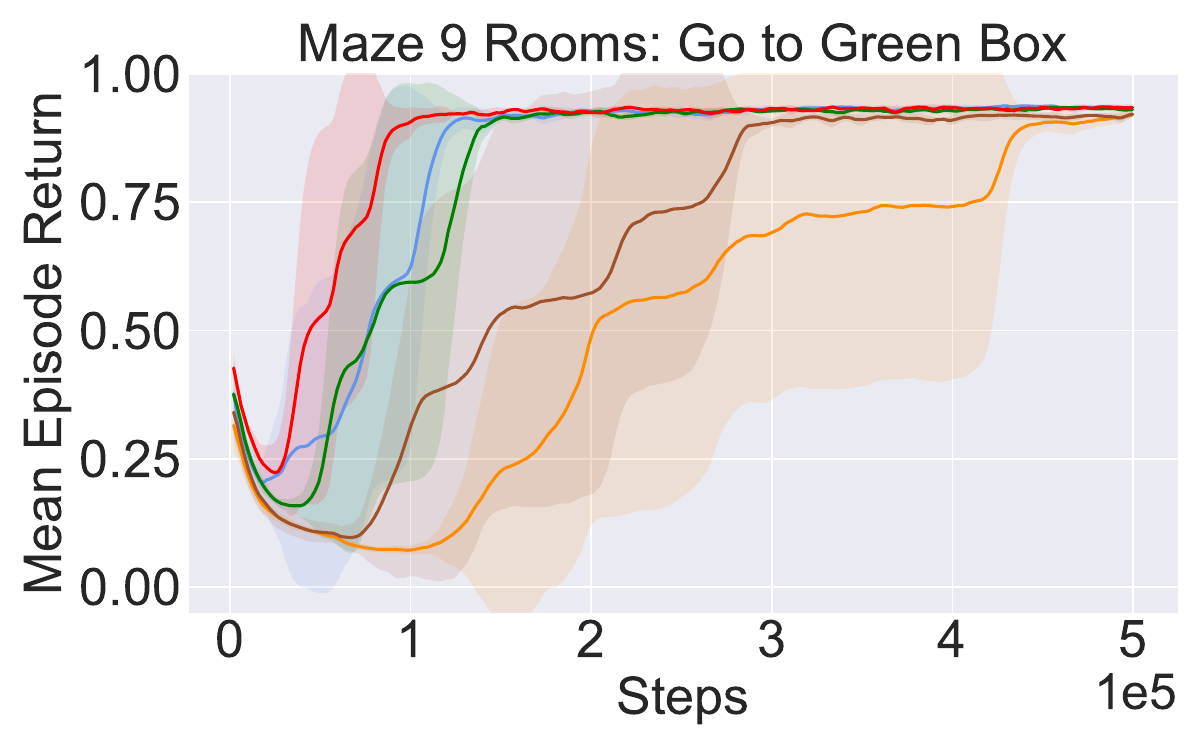}
\label{fig:results:minigrid_7}}
\hfill
\subfigure[]{
\includegraphics[height=0.30\columnwidth]{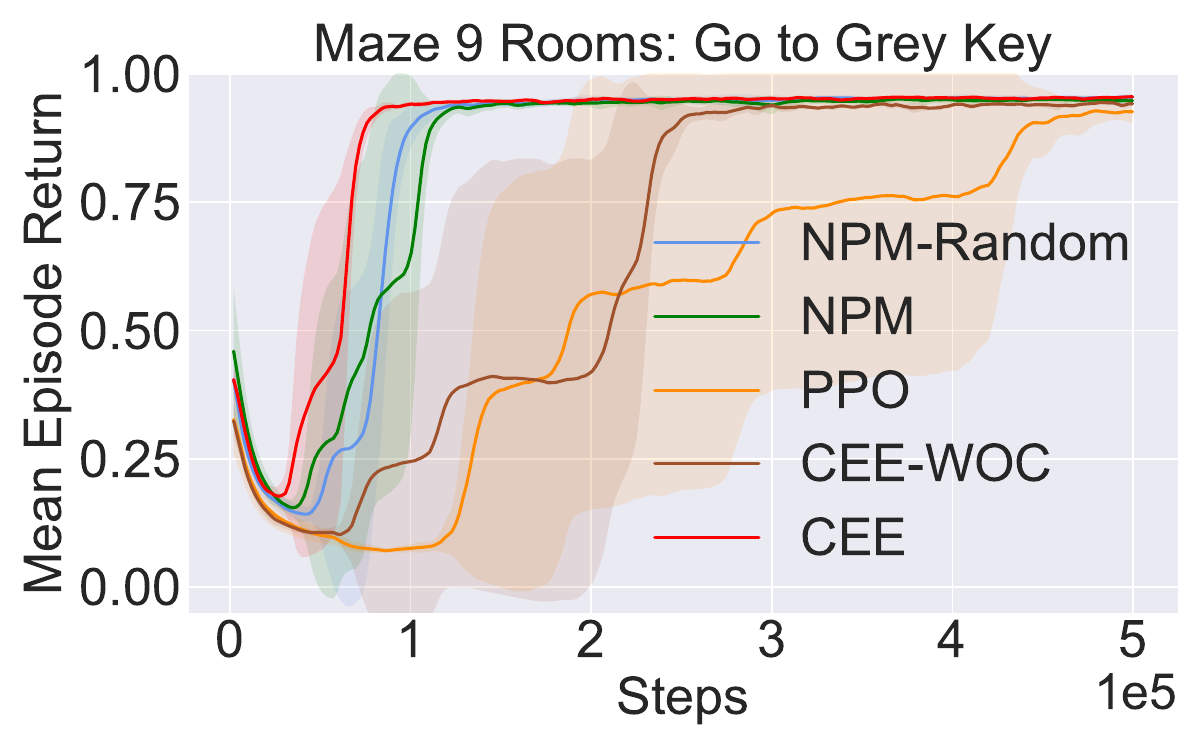}
\label{fig:results:minigrid_8}}
\hfill
\caption{Results on 8 different tasks in MiniGrid environment.}
\label{fig:results:minigrid}
\end{figure*}

\begin{figure*}[!ht]
\centering
\subfigure[]{
\includegraphics[height=0.30\columnwidth]{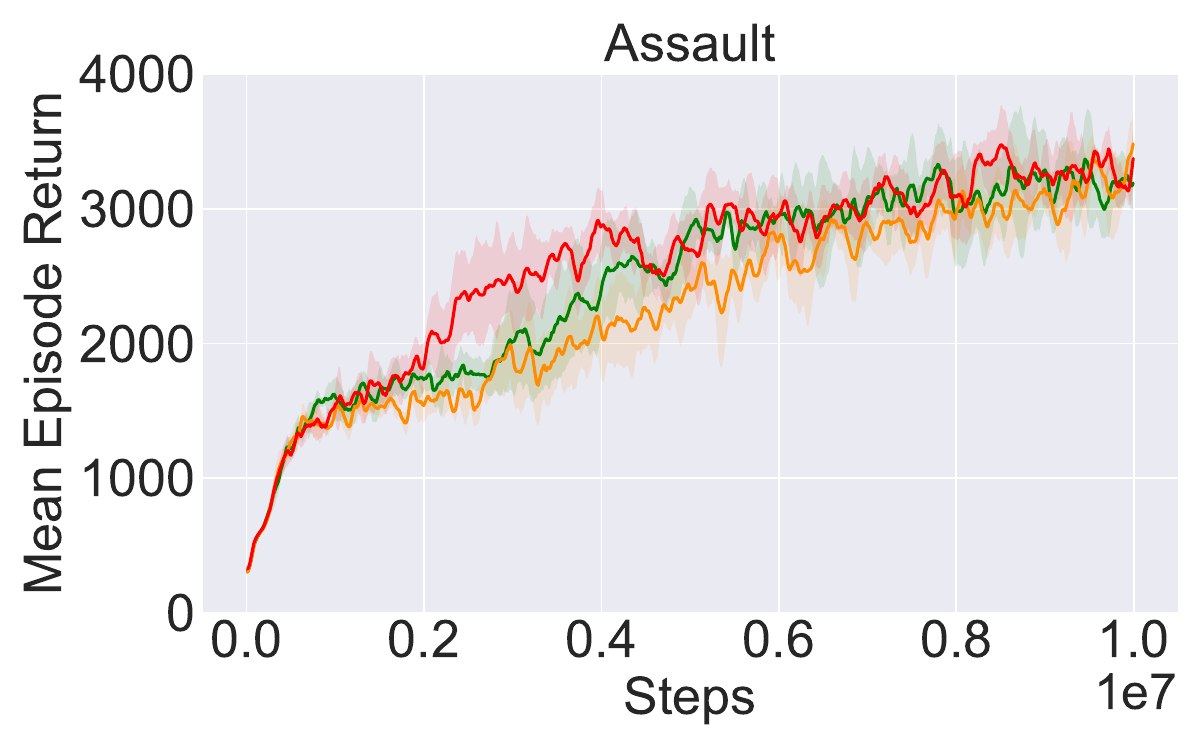}
\label{fig:results:assault}}
\hfill
\subfigure[]{
\includegraphics[height=0.30\columnwidth]{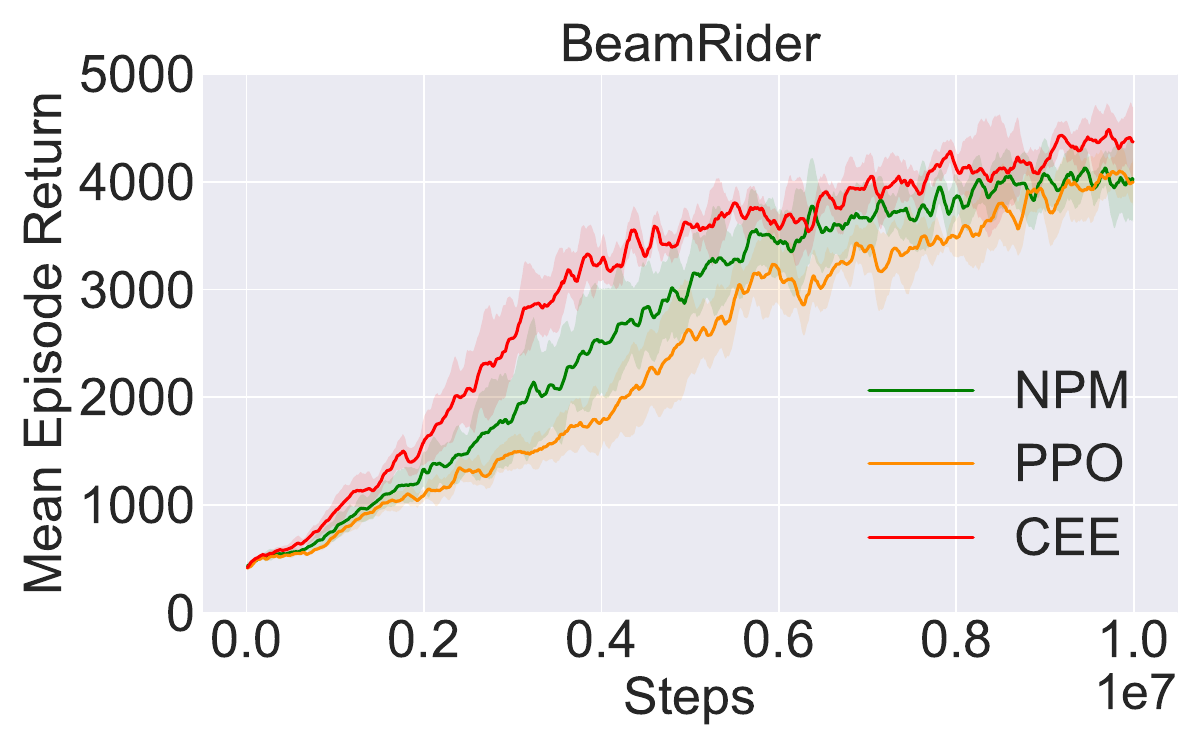}
\label{fig:results:beamrider}}
\hfill
\subfigure[]{
\includegraphics[height=0.30\columnwidth]{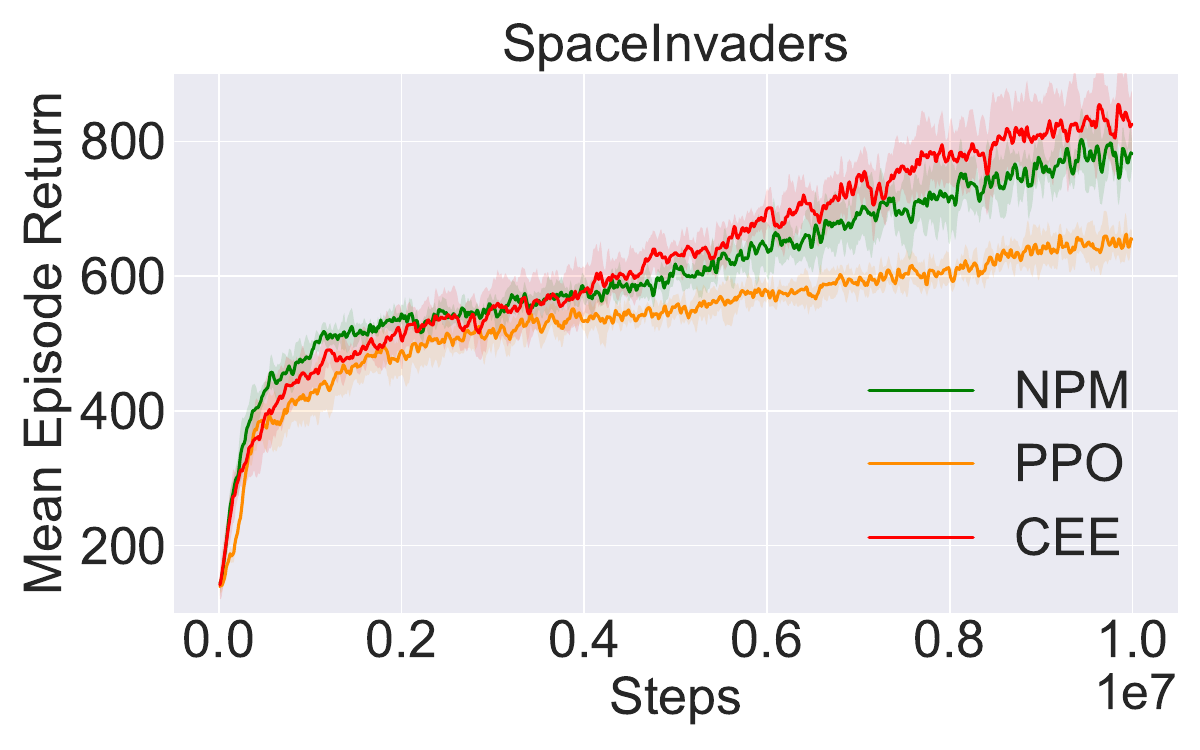}
\label{fig:results:spaceinvaders}}
\hfill
\subfigure[]{
\includegraphics[height=0.30\columnwidth]{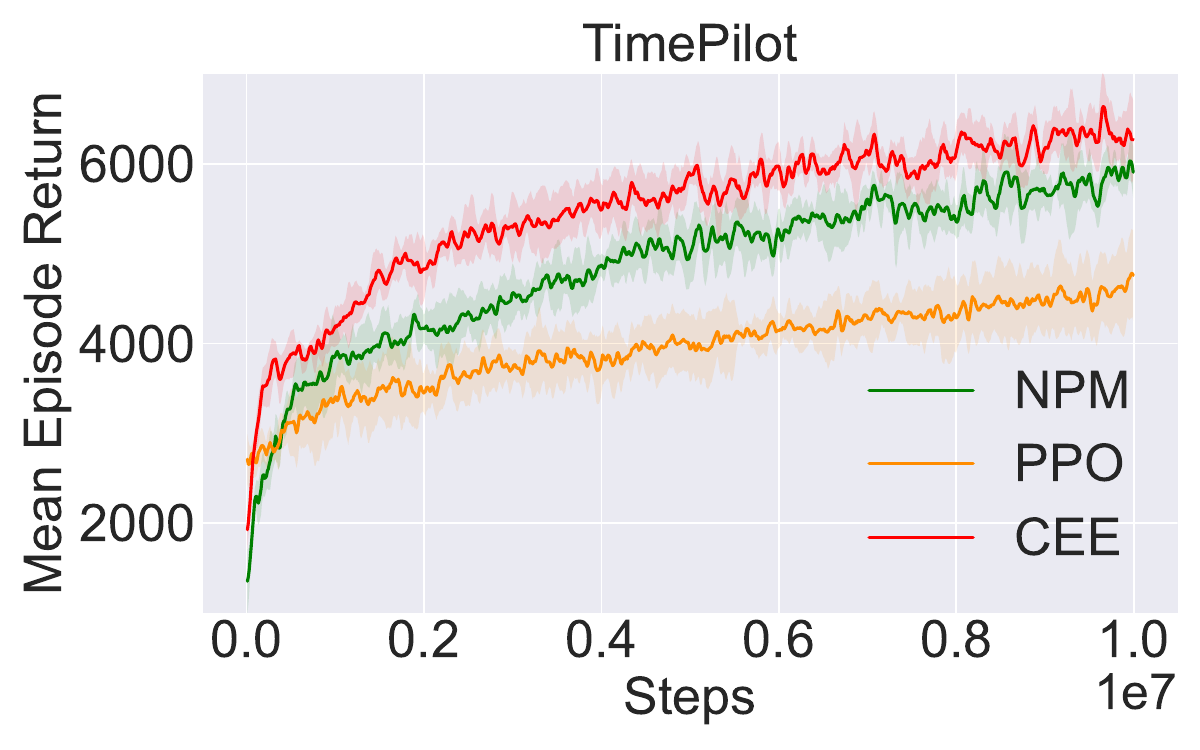}
\label{fig:results:timepilot}}
\caption{Results on part of tasks in Atari 2600 environment.}
\label{fig:results:Atari}
\end{figure*}

\section{Experiments}

In this section, we evaluate our method and compare it with baselines in simulated environments that contain redundant actions. 

\subsection{Simulated Environments}

\subsubsection{Maze}

As a simple case, we tested our method in the Maze environment with continuous state~\cite{chandak2019learning}. As illustrated in Figure~\ref{fig:Maze}, the environment features $n$ equally spaced actuators around the agent, each of them can be turned on or off. The agent will move in the direction pointed by the turned-on actuators, and the final moving direction is the sum of the direction vectors of the turned-on actuators. The size of the action space is exponentially related to the number of actuators, that is $|\mathcal{A}|=2^n$. The agent will only receive a reward if it reaches the goal (star) within 150 steps in each episode. Maze can set a finite and large number of discrete actions. 

\subsubsection{MiniGrid}

We also evaluate our method on MiniGrid~\cite{chevalier2024minigrid}, which is a set of lightweight and fast grid world environments with OpenAI gym interfaces. The environment is partially observable, with a grid size of $N \times N$. Each grid can have at most one object, that is, $0$ or $1$ object. The possible objects are wall, floor, door, key, ball, and goal. Some objects of the same kind may contain different colors. The agent can only pick up one object (key, ball, and box) at most. To open a locked door, the agent must use a key of the same color as the door. The action space is $\mathcal{A} = \{\text{turn left, turn right, forward, pickup, drop, toggle, done}\}$. If the agent performs a redundant action, its state remains unchanged. For example, if there is no object to be picked up in front of the agent, but the agent performs the action of picking up, then the action is redundant and the agent maintains its original state. We need to make the agent avoid such actions in the process of decision-making. The agent is tasked with completing different objectives across various environments. More details about the tasks of MiniGrid environment are described in Appendix.

\subsubsection{Atari 2600}

To further evaluate the effectiveness of our proposed method, we conduct experiments in the Atari 2600 environment, a widely used benchmark in DRL~\cite{bellemare13arcade}. In our experiments, we enable the full action space mode, where the agent can select from all possible actions in the game, regardless of whether some actions are redundant or ineffective in specific states. This setup significantly increases the complexity of the action space. We aim to test the ability of our method to efficiently reduce the action space and identify causal actions while maintaining high performance.

\subsection{Baselines}

We compare our method with the following baselines: (1) PPO~\cite{schulman2017proximal}. PPO is a common on-policy reinforcement learning algorithm that uses the policy gradient method to optimize the policy. (2) NPM~\cite{Zhong_Yang_Zhao_2024}, which is one of the most recently proposed algorithms. NPM first pre-trains the similarity factor model and classifies and eliminates redundant actions in the second phase. (3) NPM-Random. We modified NPM to randomly sample an action for each categorized action category. NPM samples the action with the default index provided by the environment within each category. We replaced the original reward with curiosity reward(e.g.,$R(S,A)=N(S,A)^{-\frac{1}{2}}$)~\cite{pathak2017curiosity} for complex environments during pre-training.

In addition, we also test our method without action classification, called CEE without classification (CEE-WOC). CEE-WOC directly calculates the causal effects of actions defined in Definition~\ref{definition:causal_effects_of_actions} and selects actions within the approximate causal action space in Definition~\ref{definition:approximate-calsal-effects}.

\subsection{Experiment Results}

\begin{figure*}[t]
\centering
\includegraphics[width=0.9\textwidth]{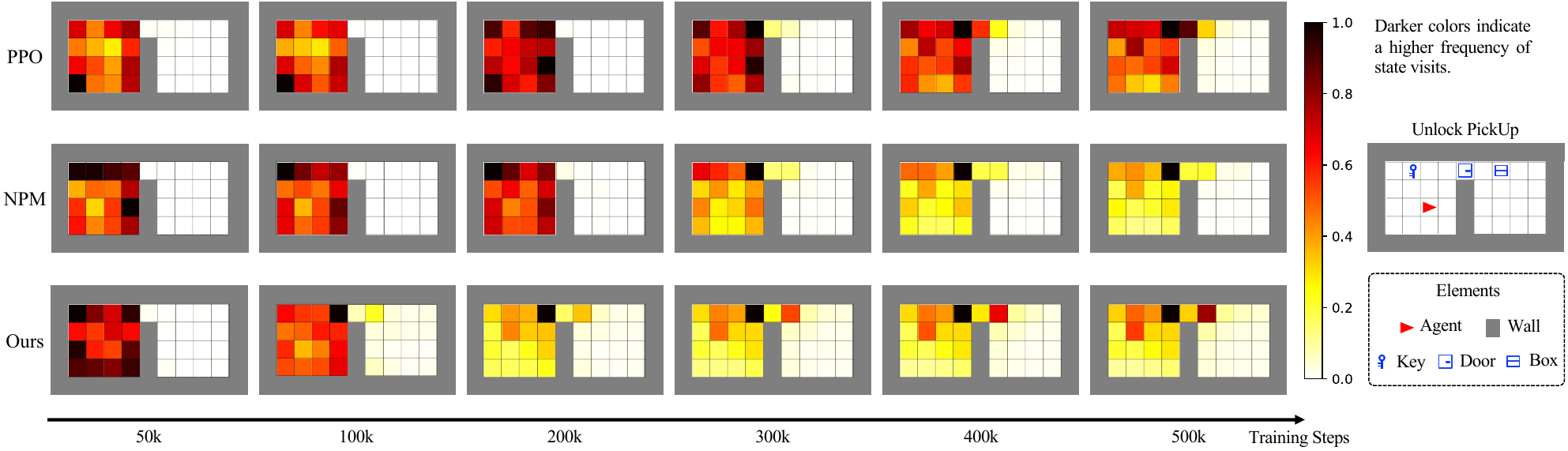}
\caption{The comparison of state visitation heatmap for ``Unlock PickUp" task. The agent's goal is to first find and pick up the key, then use the key to open the door. Once the door is open, the task is complete when the agent retrieves the box.}
\label{fig:heatmap}
\end{figure*}

The experimental settings are detailed in the Appendix. All results presented below were obtained using five different random seeds to ensure robustness and reproducibility. 

\subsubsection{Results of Maze}

In the Maze environment, we evaluated our method against baseline approaches using action sets $\mathcal{A}$ of size $|\mathcal{A}| \in \{64, 128\}$. As shown in Figure~\ref{fig:results:maze_64_actions} and~\ref{fig:results:maze_128_actions}, the experimental results demonstrate that CEE consistently outperforms the baselines across both configurations. Specifically, in the Maze with 64 actions, CEE achieves a significant performance advantage, highlighting its efficiency in optimizing smaller action spaces. Moreover, in the more challenging Maze with 128 actions, CEE-WOC also delivers impressive results, outperforming NPM and PPO even without leveraging action classification. These findings underscore the robustness and adaptability of our method in effectively handling both moderate and large-scale action spaces. 

\subsubsection{Results of MiniGrid}

In the MiniGrid environment, we compared CEE with other baseline algorithms across 8 different tasks. In the experiments in Figures~\ref{fig:results:minigrid_3}-\ref{fig:results:minigrid_8}, curiosity-driven rewards were applied to all algorithms. As shown in Figure~\ref{fig:results:minigrid}, CEE consistently outperforms the baselines across various tasks. Notably, in Figures~\ref{fig:results:minigrid_4} and~\ref{fig:results:minigrid_5}, which feature complex environments with sparse rewards, the PPO algorithm struggles to achieve meaningful results. In contrast, our algorithm begins to improve significantly around 300k steps and achieves the fastest convergence. Even with curiosity-driven rewards in Figures~\ref{fig:results:minigrid_5}, CEE consistently outperforms all baselines, demonstrating its robustness and effectiveness.

Interestingly, NPM-Random shows improved performance compared to NPM, highlighting that the selection of actions within specific categories has a notable impact on performance. Furthermore, CEE-WOC outperforms PPO in most environments, including tasks such as ``PutNextLocal", ``Unlock Pickup", and ``Maze 9 Rooms: Go to Blue Key", further validating the effectiveness of our method. However, for tasks such as ``Maze 4 Rooms: Go to Grey Key" and ``Maze 4 Rooms: Go to Red Ball", both CEE-WOC and PPO fail to achieve satisfactory performance, while other methods perform better. This result indicates that action classification is beneficial.

\subsubsection{Results of Atari}

In Atari 2600 environment, we compare our method with two baselines: PPO and NPM. Based on the results shown in Figure~\ref{fig:results:Atari}, our method consistently outperforms the baseline methods across all tested tasks, including Assault, BeamRider, and TimePilot. The CEE curve demonstrates significantly higher mean episode returns throughout the training process, indicating both superior performance and improved learning efficiency. Notably, CEE converges faster than the baseline methods, with smoother and more stable progress, even in challenging environments. For instance, in BeamRider and TimePilot, CEE achieves not only higher peak returns but also maintains performance stability over extended training steps. These results highlight the robustness and adaptability of our method, showcasing its ability to optimize decision-making and reduce the action space effectively in diverse and complex scenarios.

\subsubsection{Visitation Locations Heatmap}

We compared PPO, NPM, and our method within the ``Unlock PickUp" environment and visualized the results using a heatmap. The agent's task is to pick up the box in the adjacent room, but it first needs to retrieve the key to unlock the locked door. As shown in Figure~\ref{fig:heatmap}, at step 50k, our method demonstrates a deeper and more uniform color distribution, indicating that the agent has thoroughly explored the environment. The agent also spends more time near the doorway, highlighting its recognition of this critical area. By step 100k, the agent successfully enters another room and completes the task. In contrast, other algorithms are still exploring the first room even at step 200k. By step 500k, our algorithm exhibits high stability, with the agent’s movement path clearly defined.

\subsubsection{Ablation Study}

\begin{figure}[!ht]
\centering
\subfigure[]{
\includegraphics[height=0.29\columnwidth]{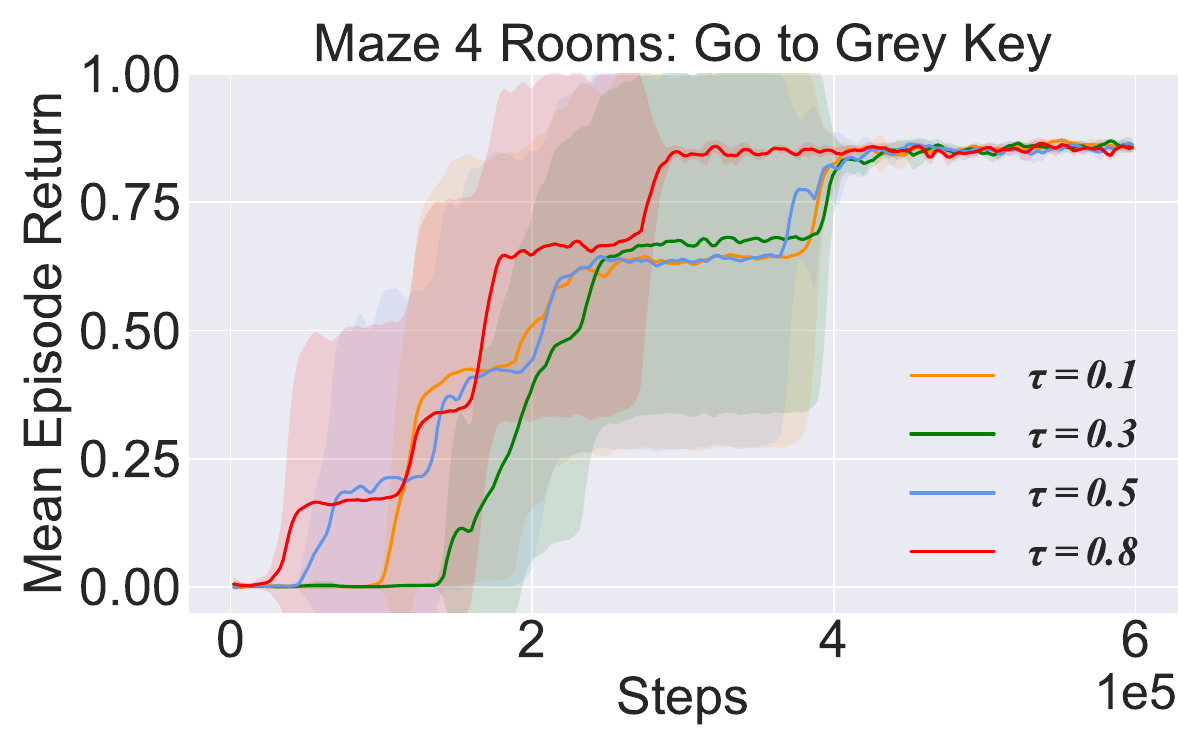}
\label{fig:results:ablation:threshold=0.3}}
\hfill
\subfigure[]{
\includegraphics[height=0.29\columnwidth]{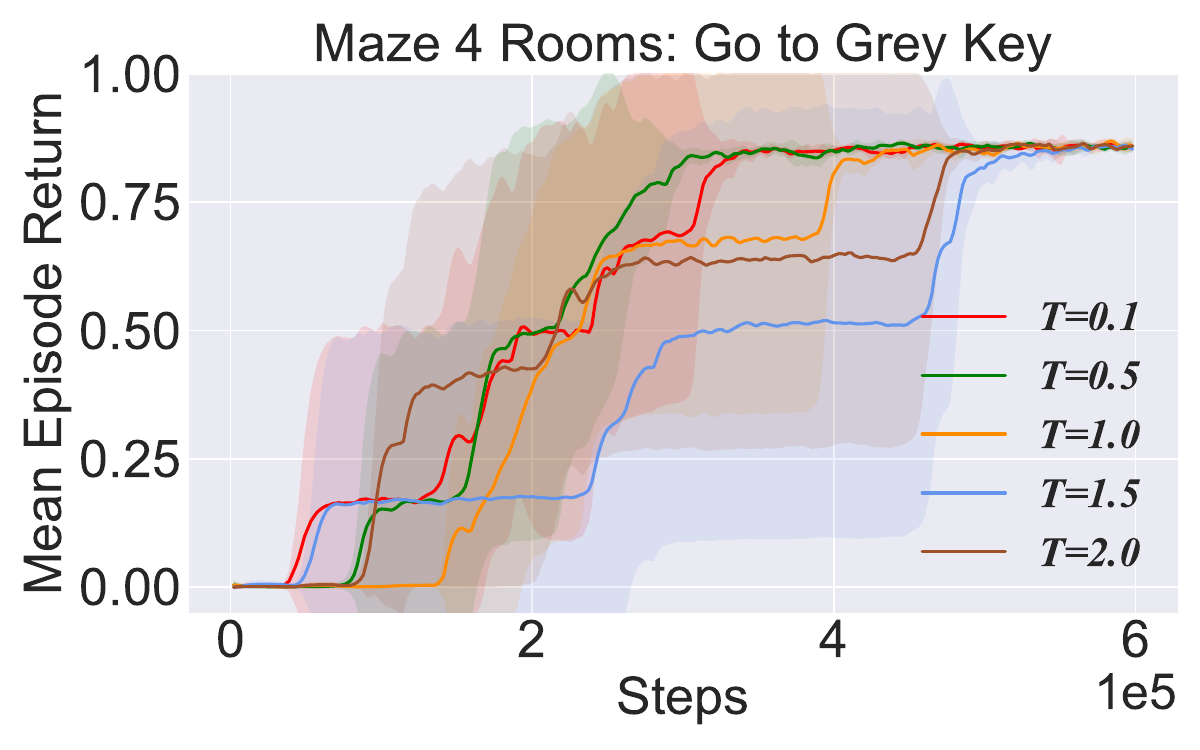}
\label{fig:results:ablation:temperature=1.0}}
\caption{(a) Ablation study with the $\tau \in \{0.1, 0.3, 0.5, 0.8\}$; (b) Ablation study with the $\tau=0.3$ and $T \in \{0.1, 0.5, 1.0, 1.5,2.0\}$.}
\label{fig:results:ablation}
\end{figure}

\paragraph{How to select $\tau$}
To achieve optimal performance, selecting the appropriate threshold is important for our method. When the $\tau$ is too low, the relative causal effects of actions are not sufficiently captured. On the other hand, if the $\tau$ is too high, the relative causal effects are overly emphasized, causing the agent to exploit the environment excessively while lacking efficient exploration. We conducted an ablation study with the $\tau\subset\{0.1, 0.3, 0.5, 0.8\}$. The results as shown in Figure~\ref{fig:results:ablation:threshold=0.3}, indicate that CEE reaches its best performance at the $\tau=0.8$. 

\paragraph{Temperature $T$} In our experiments, the selection of the temperature parameter $T$ also impacts performance. We conducted an ablation study, setting $\tau=0.3$ and $T\subset\{0.1, 0.5, 1.0, 1.5,2.0\}$ in the experiments shown in Figure~\ref{fig:results:ablation:temperature=1.0}. We observed that the algorithm performs better when $T<=1$, indicating that within an action category, actions with higher causal effects contribute positively to the algorithm's performance. 

\section{Conclusion}

In this paper, we proposed a new method to improve the performance of DRL by reducing the action space through CEE. We introduced a causal graphical model to represent state transitions and identify potentially redundant actions based on their causal effects on the next state. By pre-training an N-value network alongside an inverse dynamics model, we provided a direct and efficient evaluation of causal effects. Furthermore, we developed an action masking mechanism that integrates action classification with CEE, significantly improving both exploration efficiency and learning performance.
Extensive experiments conducted across diverse environments validated the effectiveness of our method, demonstrating notable improvements in learning efficiency and overall performance. This work presents a promising framework for tackling the challenges of large action spaces in DRL and paves the way for future research in action space reduction for complex decision-making problems.

\appendix

\section*{Acknowledgments}

This work is supported by the National Natural Science Foundation of China (Nos. 62236002, 62495083, 61921004).

\bibliographystyle{named}
\bibliography{ijcai25}

\newpage
\onecolumn

\clearpage
\begin{center}
    {\LARGE\bfseries Appendix}
\end{center}

\section*{Proof of Theorem~\ref{theorem:causal_effects}}
\label{appendix:proof:theorem-1}

\setcounter{theorem}{0}

\begin{theorem}
    Consider a causal graphical model $\mathcal{G}$ induced by a set of nodes $\mathcal{V}=\{S, \pi, A, S'\}$ and the set of edges $\mathcal{E}$. $A \in \mathcal{V}$ is a cause of $S'$ given $S$ (i.e. there is an edge from $A$ to $S'$), if and only if $C^{\pi}(A | S \rightarrow S') > 0$, otherwise, $C^{\pi}(A | S \rightarrow S') = 0$. 
\end{theorem}

\begin{proof}

According to the definition of causal effects in Eq.(\ref{eq:definition_of_causality}),
\begin{align}
    C^{\pi}(A|S \rightarrow S') &= D_{KL}(P(S'|S, A) \| P^{\pi}(S'|S)) \nonumber  \\
    &= \Sigma_{S' \in \mathcal{S}}{P(S'|S, A)\log{\frac{P(S'|S, A)}{P^{\pi}(S'|S)}}}
\end{align}
On the basis of Definition~\ref{definition:causal-action}, if $A$ is not a causal action of $S'$ in $S$, then $\forall \pi_1, \pi_2$, and $\pi_1 \neq \pi_2$, we have
\begin{equation}
    P(S'|S, \text{do}(A:=\pi_1(A|S))) \equiv P(S'|S, \text{do}(A:=\pi_2(A|S))).
\end{equation}
Here, $\pi_1$ and $\pi_2$ represent two different policies, and $\pi_1 \neq \pi_2$ ensures that the actions are sampled under different policies to distinguish causal effects.
That means, if $A$ is not a causal action, then $P(S'|S, A) = P^{\pi}(S'|S), \forall \pi$. As a result, 

\begin{equation}
	\log{\frac{P(S'|S, A)}{P^{\pi}(S'|S)}} = 0 \Rightarrow C^{\pi}(A|S \rightarrow S') = 0.
\end{equation}

On the other hand, $C^{\pi}(A|S \rightarrow S') = 0 \iff D_{KL}(P(S'|S, A) \| P^{\pi}(S'|S))=0 \iff P(S'|S, A) \equiv P^{\pi}(S'|S)$. Then, $\forall \pi_1, \pi_2$, and $\pi_1 \neq \pi_2$, $P(S'|S, \text{do}(A:= \pi_1(A|S))) \equiv P(S'|S, \text{do}(A:= \pi_2(A|S))) \equiv P^{\pi}(S'|S)$. That precisely matches the Definition~\ref{definition:causal-action}. Hence, action $A$ causes no effects on $S'$ in state $S$. Intuitively, if $P(S'|S, A)$ equals $P^\pi(S'|S)$ for all policies $\pi$, this implies that the action $A$ does not alter the distribution of $S'$, hence it has no causal effect.

In conclusion, since KL divergence is non-negative, action $A$ is a causal action from state $S$ to $S'$ if and only if $C^{\pi}(A|S \rightarrow S') > 0$. Inversely, action $A$ causes no effect on state $S'$ with state $S$ if and only if $C^{\pi}(A|S \rightarrow S') = 0$. That completes the proof.

\end{proof}

\section*{Proof of Lemma~\ref{lemma:causal_action_effect_estimation}}

\setcounter{lemma}{0}

\begin{lemma}
    Define the N-value for $A_i, A_j \in \mathcal{A}$ as
    \begin{equation}
        N(S, A_i, A_j) = \mathbb{E}_{S' \sim P(\cdot|S, A_i)}{\left[ \log{\frac{P^{\pi}(A_j|S, S')}{\pi(A_j|S)}} \right]},
    \label{definition:N-value}
    \end{equation}
    then, the calculation in Eq.(\ref{eq:definition_of_causality}) can be expressed as
    \begin{equation}
        C^{\pi}(A | S \rightarrow S') = N(S, A, A).
        \label{eq:causal-effects-estimation-via-N-values}
    \end{equation}
\end{lemma}

\begin{proof}
According to Eq.(\ref{eq:definition_of_causality}), 
\begin{align}
    C^{\pi}(A | S \rightarrow S') &= D_{KL}\left(P(S'|S, A) \| P^{\pi}(S'|S)\right) \nonumber \\
        &= \mathbb{E}_{S' \sim P(\cdot|S, A)}\left[\log{ \frac{P(S'|S, A)}{P^{\pi}(S'|S)}} \right].
\end{align}

Using Bayesian rules, we have

\begin{align}
    P(S'|S, A) &= \frac{P^{\pi}(S, A, S')}{P^{\pi}(S, A)} \\
    &= \frac{P^{\pi}(A | S, S') \times P^{\pi}(S, S')}{P^{\pi}(S, A)},
\label{proof:causality_0}
\end{align}
and
\begin{equation}
    P^{\pi}(S'|S) = \frac{P^{\pi}(S, S')}{P(S)}.
\label{proof:causality_1}
\end{equation}
Then, it can be derived that 
\begin{align}
    \frac{P(S'|S, A)}{P^{\pi}(S'|S)} = \frac{P^{\pi}(A | S, S')}{P^{\pi}(S, A)/P(S)} = \frac{P^{\pi}(A | S, S')}{\pi(A|S)}.
\end{align}
Hence, we have
\begin{equation}
    C^{\pi}(A | S \rightarrow S') = \mathbb{E}_{S' \sim P(\cdot|S, A)}\left[ \log{\frac{P^{\pi}(A | S, S')}{\pi(A|S)}} \right],
\end{equation}
which aligns with the definition of $N$-value in Eq.(\ref{definition:N-value}) when $A_i=A_j=A$, and concludes the proof.

\end{proof}

\section*{Algorithm}

Algorithm~\ref{algorithm:cee} outlines our proposed Causal Effect Estimation (CEE) approach for Deep Reinforcement Learning. The method takes as input an initial policy $\pi_{\theta}$, an inverse dynamics model $P_{\psi}^{\mathrm{inv}}$, and an N-value network $N_{\phi}$, each with corresponding initial parameters. After specifying the hyper-parameters $\epsilon$, $\tau$, $T$, and the update interval $K$ for the N-value network, the training proceeds in two phases.

In Phase 1, we pre-train both the inverse dynamics model and the N-value network, following the framework established in \cite{Zhong_Yang_Zhao_2024}. We repeatedly roll out trajectories using the current policy to collect experience into a replay buffer $\mathcal{D}$. From $\mathcal{D}$, samples are drawn to update the inverse dynamics model parameters $\psi$. Every $K$ iterations, we also update the N-value network parameters $\phi$. When the task is deemed sufficiently complex, we incorporate curiosity-driven rewards to refine the policy parameters. 

Moving on to Phase 2, we optimize the policy by estimating causal effects and reducing the action space accordingly. Concretely, at each iteration, we first compute a similarity factor matrix $M(s_t,\cdot,\cdot)$ using the N-value network. We then cluster the available actions into groups, within which we measure causal effects based on Equation (\ref{eq:causal-effects-estimation-via-N-values}). After normalizing these causal effects (Equation (\ref{eq:definition_of_relative_causality})), we identify a minimal causal action subset $\mathcal{A}^{\text{min}}_{S \to S', \tau}$ (Equations (\ref{definition:minimal-calsal-effects-eq1}) and (\ref{definition:minimal-calsal-effects-eq2})) and derive a masking vector $\boldsymbol{m}$. An action $a_m$ is sampled from the resulting masked policy $\pi_{\theta}^m$, executed in the environment, and the experience is stored back into $\mathcal{D}$. Finally, we update $\pi_{\theta}$ using standard reinforcement learning methods. The algorithm returns the optimized policy parameters $\theta$ and the updated N-value network parameters $\phi$.

\begin{algorithm}
\caption{Causal Effect Estimation (CEE) for Deep Reinforcement Learning}

\begin{algorithmic}[1]
\STATE \textbf{Inputs:} Initialize policy $\pi_\theta$, inverse dynamics model $P^{\text{inv}}_\psi$, N-value network $N_\phi$ with $\theta_0, \psi_0, \phi_0$.
\STATE \textbf{Hyper-Parameter:} Set the hyper-parameter $\epsilon$, $\tau$, $T$, and N-value network update interval $K$.

\STATE \textbf{Phase 1: Pre-train Inverse Dynamics Model and N-value Network}
\FOR{iteration $i = 0, 1, 2, \dots$}
    \STATE Rollout with Policy (random initialized) and store data into buffer $\mathcal{D}$
    \STATE Sample transitions from buffer $\mathcal{D}$
    \STATE Update inverse dynamics model parameters $\psi$
    \IF{$i$ \% $K == 0$}
        \STATE Update N-value network parameters $\phi$
    \ENDIF
    \IF{task is complex}
        \STATE Update policy parameters based on the curiosity reward
    \ENDIF
\ENDFOR

\STATE \textbf{Phase 2: Optimize Policy based on the Causal Effect Estimation}
\FOR{$t = 0, 1, 2, \dots$}
    \STATE Evaluate the similarity factor matrix $M(s_t, \cdot, \cdot)$ using N-value network $N_\phi$
    \STATE Cluster actions into several groups $\cup_{k}{\mathcal{A}^{k}_{S, \epsilon}} = \mathcal{A}$ based on Equation (\ref{eq:action-cluster}) and (\ref{eq:action-cluster-constraints})
    \STATE Evaluate the causal effects of actions $C^{\pi}(A | S \rightarrow S')$ in each group based on Equation (\ref{eq:causal-effects-estimation-via-N-values})
    \STATE Normalize the causal effects and get the relative causal effects based on Equation (\ref{eq:definition_of_relative_causality})
    \STATE Build the minimal causal action space $\mathcal{A}^{\text{min}}_{S \to S', \tau} \subseteq \mathcal{A}$ based on Equation (\ref{definition:minimal-calsal-effects-eq1}) and (\ref{definition:minimal-calsal-effects-eq2})
    \STATE Get the masking vector $\boldsymbol{m}=[m_1, \dots, m_N]$ based on Equation (\ref{eq:masking-approximal-causal-action-space})
    \STATE Sample an action $a_m \in \mathcal{A}^{\text{min}}_{S \to S', \tau}$ from the masked policy $\pi^{m}_{\theta}$ that is defined in Equation (\ref{equation:masked_policy})
    \STATE Execute action $a_m$ and store data into buffer $\mathcal{D}$
    \STATE Optimize $\pi_\theta$ with experiences in buffer $\mathcal{D}$ using reinforcement learning algorithm
\ENDFOR
\STATE \textbf{Return:} Policy parameters $\theta$ and N-value network parameters $\phi$
\end{algorithmic}
\label{algorithm:cee}
\end{algorithm}

\newpage

\section*{MiniGrid Environment}

MiniGrid is a lightweight, grid-based environment specifically designed for evaluating DRL algorithms. In our experiments, we selected eight diverse tasks to validate the efficiency of our method: PutNextLocal, Unlock Pickup, Maze 4 Rooms: Go to Green Ball, Maze 4 Rooms: Go to Green Key, Maze 4 Rooms: Go to Red Ball, Maze 9 Rooms: Go to Blue Key, Maze 9 Rooms: Go to Green Box, and Maze 9 Rooms: Go to Grey Key. These tasks present varying levels of complexity and action space challenges, making them ideal benchmarks for assessing the effectiveness of our approach. The layouts of the maps for these eight tasks are illustrated in Figure~\ref{fig:MiniGrid_environments}. For all tasks, a reward of `$1 - 0.9 * (\text{step\_count} / \text{max\_steps})$' is given for success, and `0' for failure.

\begin{figure*}[h]
\centering
\subfigure[]{
\includegraphics[height=0.22\columnwidth]{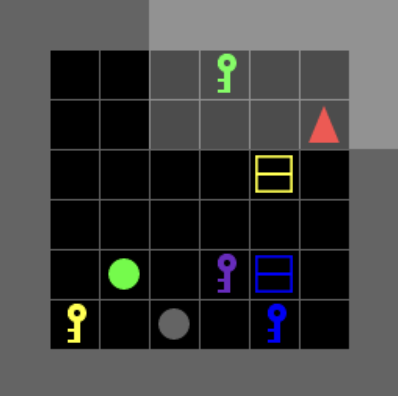}
\label{fig:MiniGrid_environments:PutNextLocal}
}
\hspace{0.05\linewidth}
%\hfill
\subfigure[]{
\includegraphics[height=0.22\columnwidth]{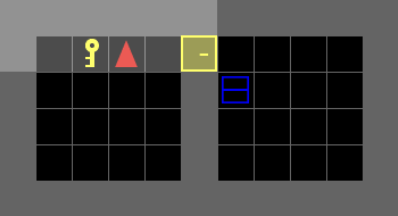}
\label{fig:MiniGrid_environments:UnlockPickUp}
}
\hspace{0.05\linewidth}
%\hfill
\subfigure[]{
\includegraphics[height=0.25\columnwidth]{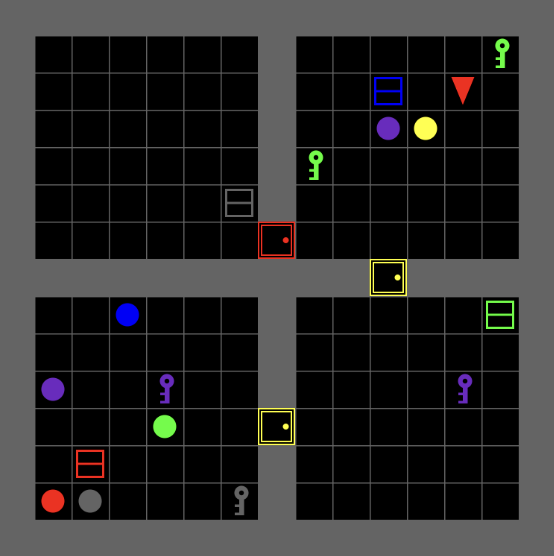}
\label{fig:MiniGrid_environments:Maze4Rooms}
}
\hspace{0.05\linewidth}
%\hfill
\subfigure[]{
\includegraphics[height=0.25\columnwidth]{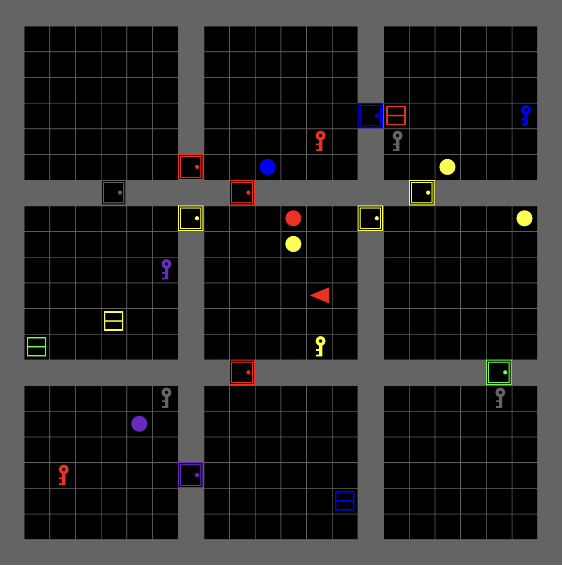}
\label{fig:MiniGrid_environments:Maze9Rooms}
}
\caption{Illustrations of MiniGrid environment: (a) PutNextLocal; (b) Unlock Pickup; (c) Maze 4 Rooms; (d) Maze 9 Rooms.}
\label{fig:MiniGrid_environments}
\end{figure*}

\subsection*{Task Description}

\noindent
\textbf{PutNextLocal:}
In the PutNextLocal task, the agent's objective is to place an object of a specified color and type adjacent to another object of a different specified color and type within a single room with $6 \times 6$ grids (see Figure~\ref{fig:MiniGrid_environments:PutNextLocal}). The objects can be one of five colors: ``red", ``green", ``purple", ``yellow", or ``grey", and one of three types: ``ball", ``box", or ``key". 

\noindent
\textbf{Unlock Pickup:}
In the Unlock Pickup task, the agent must pick up a box placed in a locked room that can only be accessed using a key (see Figure~\ref{fig:MiniGrid_environments:UnlockPickUp}). To complete the task, the agent first needs to locate and pick up the key to unlock the door, granting access to the other room. Once inside, the agent explores the opened room, finds the box, and finally picks it up. This task is particularly challenging due to its complexity, as it consists of multiple subtasks, leading to sparser rewards.

\noindent
\textbf{Maze 4 Rooms:}
We selected three different tasks in the map of ``Maze 4 Rooms": ``Go to Green Ball", ``Go to Grey Key", and ``Go to Red Ball" (see Figure~\ref{fig:MiniGrid_environments:Maze4Rooms}). These three tasks are highly challenging as the agent must navigate from the one room to another room while sequentially finding and passing through the doors to adjacent rooms. The agent continues this process of entering rooms, searching for the target, or moving to the next room until it reaches the target location, thereby completing the task. However, the agent may repeatedly enter the wrong rooms, leading to a failure to locate the target. This demands a high level of exploration efficiency for the agent. Without suppressing redundant actions, the exploration policy is hard to find the target within limited steps.

\noindent
\textbf{Maze 9 Rooms:}
We also selected another three different tasks in the map of ``Maze 9 Rooms": ``Go to Blue Key", ``Go to Green Box", and ``Go to Grey Key" (see Figure~\ref{fig:MiniGrid_environments:Maze9Rooms}). These three tasks are similar to those in the ``Maze 4 Rooms" map but are potentially more challenging. With an increased number of rooms, the agent faces a significantly larger exploration space when the target's location is far from the agent's starting position, resulting in even sparser rewards. In this scenario, a reduced action space becomes crucial for the agent to efficiently navigate to the target. During the search process, the agent only needs to focus on movement-related actions, such as ``left", ``right", and ``forward", while other actions can be ignored. Repeatedly attempting irrelevant actions not only fails to contribute to task completion but also increases the number of ineffective trials, making it difficult for the experience replay buffer to collect useful data for policy training.

\subsection*{Additional Results}

Table~\ref{table:testing_results_minigrid} shows the testing results of trained policies. Our method, CEE, consistently achieves superior performance or matches the best-performing baselines across all tasks. Notably, in tasks such as ``PutNextLocal", ``Unlock Pickup", and the ``Maze 9 Rooms" scenarios, CEE demonstrates the highest or comparable rewards (e.g., 0.85, 0.90, and up to 0.96), showcasing its robustness and efficiency in diverse and complex environments.

The primary difference between CEE-WOC and CEE is the absence of the action classification mechanism in CEE-WOC. A comparison between CEE-WOC and PPO highlights the significant advantage of incorporating the causal effect estimation mechanism, demonstrating its utility in filtering redundant actions. However, when comparing CEE-WOC with CEE, it becomes evident that the action classification mechanism is crucial for further improving performance. These results suggest that relying solely on causal effect estimation to filter redundant actions has limited effectiveness.

Additionally, while NPM and NPM-Random perform well in certain tasks, such as ``Maze 4 Rooms: Go to Green Ball" and ``Maze 9 Rooms: Go to Blue Key", their performance is inconsistent in environments with greater complexity or sparser rewards. A comparison between NPM and NPM-Random reveals that selecting actions within each cluster plays an important role in filtering out redundant actions, as NPM-Random outperforms NPM in most cases. Furthermore, the comparison between CEE and NPM-Random indicates that selecting actions based on causal effects achieves even better performance, further demonstrating the advantages of our method.

Overall, the results underscore that CEE effectively addresses the challenges posed by tasks with redundant action spaces and sparse rewards, delivering significant improvements in exploration efficiency and overall performance compared to the baselines.

\renewcommand{\arraystretch}{1.5}
\begin{table}[h]
\centering
\caption{The testing results on MiniGrid environments.}
\begin{tabular}{l|rl|rl|rl|rl|rl}
\toprule
Task 			 & \multicolumn{2}{c|}{CEE (Ours)} & \multicolumn{2}{c|}{CEE-WOC} & \multicolumn{2}{c|}{NPM} & \multicolumn{2}{c|}{NPM-Random} & \multicolumn{2}{c}{PPO} \\
\midrule
PutNextLocal 	 & \textbf{0.85} & \textbf{$\pm$0.0} & 0.84 & $\pm$0.0 & 0.83 & $\pm$0.0 & \textbf{0.85} & \textbf{$\pm$0.0} & 0.80 & $\pm$0.03 \\
Unlock Pickup & \textbf{0.90} & \textbf{$\pm$0.0} & 0.88 & $\pm$0.0 & \textbf{0.90} & \textbf{$\pm$0.0} & 0.89 & $\pm$0.00 & 0.85 & $\pm$0.01 \\
Maze 4 Rooms: Go to Green Ball & \textbf{0.89} & \textbf{$\pm$0.0} & 0.21 & $\pm$0.36 & \textbf{0.89} & \textbf{$\pm$0.0} & \textbf{0.89} & \textbf{$\pm$0.0} & 0.18 & $\pm$0.34 \\
Maze 4 Rooms: Go to Grey Key & \textbf{0.89} & \textbf{$\pm$0.0} & 0.0 & $\pm$0.0 & 0.88 & $\pm$0.0 & \textbf{0.89} & \textbf{$\pm$0.0} & 0.01 & $\pm$0.0 \\
Maze 4 Rooms: Go to Red Ball & \textbf{0.83} & \textbf{$\pm$0.0} & 0.11 & $\pm$0.01 & 0.80 & $\pm$0.03 & \textbf{0.83} & \textbf{$\pm$0.0} & 0.11 & $\pm$0.01 \\
Maze 9 Rooms: Go to Blue Key & \textbf{0.95} & \textbf{$\pm$0.0} & 0.94 & $\pm$0.0 & \textbf{0.95} & \textbf{$\pm$0.0} & \textbf{0.95} & \textbf{$\pm$0.0} & 0.78 & $\pm$0.26 \\
Maze 9 Rooms: Go to Green Box & \textbf{0.95} & \textbf{$\pm$0.0} & 0.94 & $\pm$0.0 & 0.94 & $\pm$0.0 & \textbf{0.95} & \textbf{$\pm$0.0} & 0.94 & $\pm$0.0 \\
Maze 9 Rooms: Go to Grey Key & \textbf{0.96} & \textbf{$\pm$0.0} & \textbf{0.96} & \textbf{$\pm$0.0} & 0.96 & $\pm$0.01 & \textbf{0.96} & \textbf{$\pm$0.0} & \textbf{0.96} & \textbf{$\pm$0.0} \\
\bottomrule
\end{tabular}
\label{table:testing_results_minigrid}
\end{table}

\section*{Atari 2600 Benchmark}

The Atari 2600 environment~\cite{bellemare13arcade}, known for its complexity, offers a diverse set of tasks and scenarios, making it a widely used benchmark for testing DRL algorithms in decision-making problems with raw image inputs. In our experiments, to evaluate the performance of different algorithms in the presence of redundant actions, we enabled the ``full\_action\_space" mode. Under this setting, the agent has access to 18 possible actions in each scenario, as shown in Table~\ref{table:full_action_space_Atari}. However, not all actions are useful for completing the tasks in practice. This setting increases the learning difficulty, as the agent has to explore with redundant actions during training. Effectively identifying and filtering out these redundant actions based on observational data is important for reducing the action space and improving training performance. Below, we provide a detailed description of the four scenarios used in our experiments.

\begin{table}[h]\scriptsize
\centering
\caption{The full action space in Atari 2600 tasks.}
\begin{tabular}{l|lllllllll}
\toprule
Action value & 0 	& 1 	& 2 & 3 & 4 & 5 & 6 & 7 & 8 \\
\midrule
Meaning 	 & Noop	& Fire	& Up	 & Right & Left & Down & UpRight & UpLeft & DownRight \\
\toprule
Action value & 9 	& 10 & 11 & 12 & 13 & 14 & 15 & 16 & 17 \\
\midrule
Meaning 	 & DownLeft	& UpFire	& RightFire	 & LeftFire & DownFire & UpRightFire & UpLeftFire & DownRightFire & DownLeftFire \\
\bottomrule
\end{tabular}
\label{table:full_action_space_Atari}
\end{table}

\noindent
\textbf{Assault:} In this task, the agent controls a vehicle that can move sideways while a large mothership circles overhead, continuously deploying smaller drones (see Figure~\ref{fig:Atari_environments:Assault}). The agent's objective is to destroy these enemies while evading their attacks. By default, the agent has access to a subset of 7 actions: \{Noop, Fire, Up, Right, Left, RightFire, LeftFire\}. To assess the algorithms' ability to identify and filter redundant actions, we enabled the full action space (the same below).

\noindent
\textbf{BeamRider:} The agent operates a spaceship traveling forward at a constant speed, which can be steered sideways across discrete positions (see Figure~\ref{fig:Atari_environments:BeamRider}). The goal is to destroy enemy ships, avoid their attacks, and dodge space debris. By default, the agent has access to a subset of 9 actions: \{Noop, Fire, Up, Right, Left, UpRight, UpLeft, RightFire, LeftFire\}.

\noindent
\textbf{SpaceInvaders:} The agent's objective is to eliminate space invaders by firing a laser cannon before they reach Earth (see Figure~\ref{fig:Atari_environments:SpaceInvaders}). The game ends if the agent loses all lives due to enemy fire or if the invaders successfully reach Earth. By default, the agent can use a subset of 6 actions: \{Noop, Fire, Right, Left, RightFire, LeftFire\}.

\noindent
\textbf{TimePilot:} In this task, the agent pilots an aircraft tasked with destroying enemy units (see Figure~\ref{fig:Atari_environments:TimePilot}). As the game progresses, the agent encounters increasingly advanced enemies equipped with futuristic technology. By default, the agent has access to a subset of 10 actions: \{Noop, Fire, Up, Right, Left, Down, UpFire, RightFire, LeftFire, DownFire\}.

\begin{figure*}[h]
\centering
\subfigure[]{
\includegraphics[height=0.24\columnwidth]{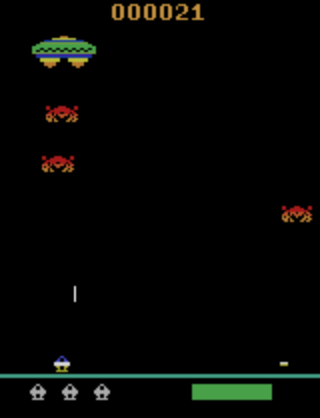}
\label{fig:Atari_environments:Assault}}
\hspace{0.05\linewidth}
%\hfill
\subfigure[]{
\includegraphics[height=0.24\columnwidth]{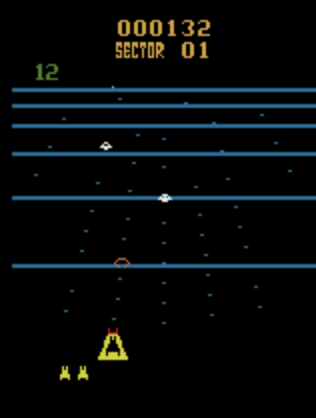}
\label{fig:Atari_environments:BeamRider}}
\hspace{0.05\linewidth}
%\hfill
\subfigure[]{
\includegraphics[height=0.24\columnwidth]{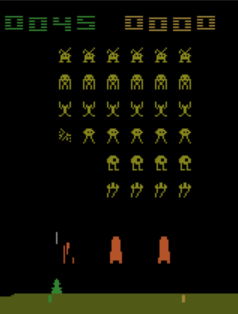}
\label{fig:Atari_environments:SpaceInvaders}}
\hspace{0.05\linewidth}
%\hfill
\subfigure[]{
\includegraphics[height=0.24\columnwidth]{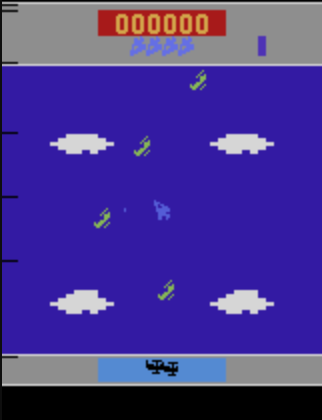}
\label{fig:Atari_environments:TimePilot}}
\caption{Part of the tasks in Atari 2600 environment: (a) Assault; (b) BeamRider; (c) SpaceInvaders; (d) TimePilot.}
\label{fig:Atari_environments}
\end{figure*}

\subsection*{Additional Results}

Table~\ref{table:additional_results_Atari} summarizes the testing performance of the trained policies across four Atari environments: Assault, BeamRider, SpaceInvaders, and TimePilot. The results demonstrate that our proposed method, CEE, consistently achieves superior performance compared to the baseline algorithms, NPM and PPO, across all tasks.

\begin{table}[h]
\centering
\caption{The testing results on Atari 2600 tasks.}
\begin{tabular}{l|rl|rl|rl}
\toprule
Task 			 & \multicolumn{2}{c|}{CEE (Ours)} & \multicolumn{2}{c|}{NPM} & \multicolumn{2}{c}{PPO} \\
\midrule
ALE/Assault-v5 	 & \textbf{3930.69} & \textbf{$\pm$195.21} & 3716.83 & $\pm$273.09 & 3695.49 & $\pm$70.14 \\
ALE/BeamRider-v5 & \textbf{4763.39} & \textbf{$\pm$212.23} & 4427.76 & $\pm$230.18 & 4390.25 & $\pm$183.86\\
ALE/SpaceInvaders-v5 & \textbf{914.97} & \textbf{$\pm$59.98} & 864.19 & $\pm$38.18 & 704.55 & $\pm$17.87\\
ALE/TimePilot-v5 & \textbf{6952.40} & \textbf{$\pm$260.72} & 6473.00 & $\pm$270.64 & 4987.00 & $\pm$458.77\\
\bottomrule
\end{tabular}
\label{table:additional_results_Atari}
\end{table}

In environments like Assault and BeamRider, which require quick reactions and efficient handling of dynamic scenarios, CEE outperforms the baselines by effectively leveraging causal effect estimation to filter redundant actions. This allows the agent to focus on meaningful actions, resulting in better decision-making and adaptability to rapidly changing game states.

In more challenging environments, such as SpaceInvaders and TimePilot, where the reward signals are sparse or the tasks involve complex interactions with diverse enemy behaviors, CEE demonstrates clear advantages. Its ability to identify and prioritize impactful actions contributes to more efficient exploration and robust performance, even under difficult conditions.

Overall, these results highlight the effectiveness of CEE in addressing the challenges posed by redundant action spaces and sparse rewards. By integrating causal effect estimation into the learning process, CEE enhances both exploration efficiency and policy robustness, making it a reliable approach for complex decision-making tasks.

\section*{Experiment Settings}

All experiments in this paper were conducted on a desktop computer equipped with an NVIDIA GeForce RTX 4090 GPU and an AMD Ryzen Threadripper PRO 5975WX CPU. The neural networks were implemented using PyTorch. 

\subsection*{Neural Network Architecture}

In this paper, all algorithms are implemented based on PPO. Therefore, the structure of the neural network setting remains the same for all experimental tasks. Below, we outline the neural network architectures used for training in our experiments. We use two hidden layers with the ReLU activation function and the size of each layer is 64 for both actor and critic networks. The input to the neural network is a 3-channel image, processed by three convolutional layers with ReLU activation. The first layer applies 32 filters of size 8×8 with a stride of 4. The second layer uses 64 filters of size 4×4 with a stride of 2. The third layer applies 64 filters of size 3×3 with a stride of 1. The output from these convolutional layers is then flattened into a 1-dimensional tensor and passed through a fully connected linear layer with 512 output units, followed by a ReLU activation.

\subsection*{Training Settings}

The hyper-parameters of our method are set as follows: Both of the learning rates of actor and critic networks are set as 0.0003. The model is trained using a batch of data with a size of 64 and a replay buffer with a size of 50000. The discount factor $\gamma=0.99$. The Adam optimizer is chosen to update the neural networks. We also apply the GAE trick with $\lambda=0.95$. In our experiment, our relative causal effects of actions threshold is set as $\tau=0.8$. A detailed summary of the hyper-parameter settings can be found in Table~\ref{table:hyper-parameters_settings}

\renewcommand{\arraystretch}{1.2}
\begin{table}[h]
\centering
\caption{Hyper-parameters settings.}
\begin{tabular}{l|l|l}
\toprule
Hyperparameter & Maze and MiniGrid & Atari \\
\midrule
\textbf{Phase 1:} & \\
Learning rate (N-value network) & 0.0003 & 0.00025 \\
Number of steps & 2048 & 128 \\
Update interval for N-value network ($K$) & 1 & 1 \\
Batch size & 64 & 256 \\
Number of epochs & 10 & 4\\
Discount factor ($\gamma$) & 0.99 & 0.99 \\
GAE factor ($\lambda$) & 0.95 & 0.95 \\
Value clip & 0.2 & 0.1 \\
Value coefficient & 0.5 & 0.5 \\
Entropy coefficient & 0.2 & 0.01 \\
Buffer size & 50000 & 100000 \\
\midrule
\textbf{Phase 2:} &	\\
Learning rate for policy network & 0.0003 & 0.00025 \\
Number of steps & 2048 & 128 \\
Batch size & 64 & 256 \\
Number of epochs & 10 & 4 \\
Discount factor ($\gamma$) & 0.99 & 0.99 \\
GAE factor ($\lambda$) & 0.95 & 0.95 \\
Value clip range & 0.2 & 0.1 \\
Value coefficient & 0.5 & 0.5 \\
Entropy coefficient & 0.2 & 0.01 \\
Masked threshold ($\epsilon$) & 0.5 & 0.5 \\
Threshold for causal effects ($\tau$) & 0.8 & 0.8 \\
Temperature for softmax ($T$) & 1.0 & 1.0\\
\bottomrule
\end{tabular}
\label{table:hyper-parameters_settings}
\end{table}

\end{document}